\documentclass[11pt]{article}

\usepackage[top=30truemm,bottom=30truemm,left=25truemm,right=25truemm]{geometry}
\usepackage{parskip}
\usepackage[utf8]{inputenc}
\usepackage[T1]{fontenc}
\usepackage{tgtermes}
\usepackage{microtype}
\usepackage{amsmath}
\usepackage{amssymb}
\usepackage{amsthm}
\usepackage{mathtools}

\usepackage[authoryear,round]{natbib}

\usepackage{xcolor}
\usepackage[hypertexnames=false,backref,pagebackref]{hyperref}
\hypersetup{
  colorlinks=true,
  linkcolor=blue,
  citecolor=blue,
}
\usepackage[nameinlink,capitalise,noabbrev]{cleveref}
\newcommand{\fixcrefthm}[1]{%
  \AddToHook{env/#1/begin}{\crefalias{theorem}{#1}}%
}
\fixcrefthm{theorem}
\fixcrefthm{lemma}
\fixcrefthm{proposition}
\fixcrefthm{corollary}
\fixcrefthm{assumption}
\fixcrefthm{definition}
\fixcrefthm{remark}

\crefname{equation}{}{}
\crefname{theorem}{Theorem}{Theorems}
\Crefname{theorem}{Theorem}{Theorems}
\crefname{lemma}{Lemma}{Lemmas}
\Crefname{lemma}{Lemma}{Lemmas}
\crefname{corollary}{Corollary}{Corollaries}
\Crefname{corollary}{Corollary}{Corollaries}
\crefname{definition}{Definition}{Definitions}
\Crefname{definition}{Definition}{Definitions}
\crefname{remark}{Remark}{Remarks}
\Crefname{remark}{Remark}{Remarks}
\crefname{assumption}{Assumption}{Assumptions}
\Crefname{assumption}{Assumption}{Assumptions}
\crefname{proposition}{Proposition}{Propositions}
\Crefname{proposition}{Proposition}{Propositions}
\Crefname{algocf}{Algorithm}{Algorithms}
\crefname{appendix}{Appendix}{Appendices}
\Crefname{appendix}{Appendix}{Appendices}

\usepackage[ruled,vlined,linesnumbered]{algorithm2e}
\SetArgSty{textnormal}
\SetKwInOut{KwRequire}{Require}

\usepackage{graphicx}
\graphicspath{{figure/}}
\usepackage[labelfont=bf]{caption}
\usepackage{enumitem}
\usepackage{booktabs}
\usepackage{pdflscape}

\usepackage{wrapfig}
\usepackage{tikz}

\theoremstyle{plain}
\newtheorem{theorem}{Theorem}[section]
\newtheorem{lemma}[theorem]{Lemma}

\newtheorem{assumption}[theorem]{Assumption}
\theoremstyle{definition}
\newtheorem{definition}[theorem]{Definition}
\theoremstyle{remark}

\renewcommand{\epsilon}{\varepsilon}
\def\E{\mathbb{E}}
\def\P{\mathbb{P}}
\def\R{\mathbb{R}}

\DeclareMathOperator*{\argmin}{arg\,min}
\DeclareMathOperator*{\diag}{diag}
\DeclareMathOperator*{\tr}{tr}
\DeclareMathOperator{\polar}{polar}
\DeclareMathOperator{\Lip}{Lip}

\DeclarePairedDelimiter{\inpr}{\langle}{\rangle}
\DeclarePairedDelimiter{\brk}{[}{]}
\newcommand{\calH}{\mathcal{H}}

\newcommand{\DReg}{\operatorname{Reg}^{[\beta]}}

\title{Muon with Finite Newton--Schulz: The Smoothing Benefit in Nonsmooth Nonconvex Optimization\thanks{Authors are listed in alphabetical order.}}

\author{
  Mingyi Li\footnote{
    The University of Tokyo;
    \texttt{mingyi-mike@g.ecc.u-tokyo.ac.jp}.
  }
  \and
  Taira Tsuchiya\footnote{
    The University of Tokyo and RIKEN;
    \texttt{tsuchiya@mist.i.u-tokyo.ac.jp}.
  }
}

\date{\today}

\begin{document}
\maketitle

\begin{abstract}
Muon has emerged as a strong optimizer for the matrix-valued parameters in large language model pretraining, approximately orthogonalizing its momentum with a few Newton--Schulz iterations.
Existing theory either replaces this iteration with the exact polar factor it approximates, or treats its finite depth as an approximation error, and thus the iteration Muon actually runs can only hurt the guarantees.
We show that finite Newton--Schulz can instead be beneficial for nonsmooth nonconvex optimization.
To this end, we analyze Muon through the online-to-nonconvex conversion, which views the update rule as an online learner and converts its regret bound into a stationarity guarantee.
The finite Newton--Schulz iteration smooths the discontinuous polar map into a Lipschitz map of the singular values, and Muon with finite Newton--Schulz can be regarded as an online learner with a smoothed spectral potential.
This smoothing is exactly what the conversion needs: we prove that a Newton--Schulz depth growing only logarithmically in the target accuracy suffices for convergence to stationary points in nonsmooth nonconvex optimization, whereas Muon with the exact-polar update may fail to converge.
The resulting sample complexity bounds match the best-known guarantees for nonsmooth nonconvex optimization and are optimal for smooth nonconvex optimization up to problem-dependent factors.
The argument extends beyond Newton--Schulz to general spectral maps with the same smoothing property.
\end{abstract}

\section{Introduction}
\label{sec:intro}

In large language model (LLM) pretraining, where most trainable parameters are matrices, Muon~\citep{jordan2024muon} has emerged as a strong alternative to the de facto standard AdamW~\citep{loshchilov19decoupled}.
Scaling studies report substantial efficiency gains~\citep{liu25scalable}, Muon-based optimizers have been used to train frontier models such as Kimi~K2 and GLM-4.5~\citep{kimi25k2,glm25glm}, and a systematic benchmark places Muon among the strongest pretraining optimizers~\citep{wen26fantastic}.
This success depends on an efficient implementation of the operation at the core of Muon, the orthogonalization of the update direction.

At each round, Muon maintains an exponential moving average $M$ of the stochastic gradients and steps along an orthogonalization of $M$: for a singular value decomposition $M=U\Sigma V^\top$, the ideal direction is the polar factor $UV^\top$, which keeps the singular subspaces of $M$ and maps every positive singular value to one~\citep{jordan2024muon,bernstein24old}.
Instead of computing $UV^\top$ through a singular value decomposition, Muon approximates it with a few iterations of Newton--Schulz, a classical scheme built from a fixed odd matrix polynomial that requires only matrix--matrix multiplications~\citep{kovarik70iterative,bjorck71iterative}.
We refer to the iteration truncated at a finite depth as finite Newton--Schulz.
This inexpensive orthogonalization is a key reason why Muon is practical at the scale of LLM pretraining.

Muon's empirical success has prompted a rapidly growing convergence theory, so far mostly for smooth objectives.
Most of it, however, replaces the Newton--Schulz loop with the exact polar factor, or, more generally, with an exact linear minimization oracle over a norm ball~\citep[\textit{e.g.,}][]{li25note,kovalev25orthogonalization,pethick25lmo,shen26convergence,chen26spectral,riabinin26gluon,sfyraki26lions}, so the finite Newton--Schulz iteration never enters the analysis.
More recent analyses account for the finite Newton--Schulz iteration, but treat its finite-depth effect as an approximation error relative to the exact-polar update~\citep{kim26convergence,shulgin26inexact,choudhury26nesterov}.
In these analyses, finite Newton--Schulz can only hurt because the guarantees improve when the iteration more closely approximates the exact polar factor.

A separate line of theoretical work removes the smoothness assumption on the objective and instead builds on online learning.
The online-to-nonconvex conversion (O2NC) of~\citet{cutkosky23optimal} lets an online learner choose the update increments of an optimizer.
The learner's regret then translates into a bound on a relaxed, Goldstein-type stationarity measure~\citep{goldstein77optimization}, the standard target in nonsmooth nonconvex optimization.
Applying this framework to matrix optimizers, \citet{jiang26adaptive} observe that Muon with the exact-polar update coincides with follow-the-leader over the operator-norm ball, an online algorithm that can suffer linear regret and thus cannot guarantee convergence to stationary points.
They therefore replace the online learner, injecting stochastic perturbations or augmenting the momentum matrix, and derive the Pion and Leon algorithms, which come with stationarity guarantees for nonsmooth objectives.
These methods, however, depart from Muon as implemented: they add ingredients that Muon does not use, and the finite Newton--Schulz iteration that Muon does use never enters their analyses.

Thus one line of work admits finite Newton--Schulz only as an error to be controlled and requires smoothness, while the other handles nonsmooth objectives but replaces the update that Muon actually performs.
This leaves the basic question open:
\begin{center}
\emph{Can the finite Newton--Schulz iteration be a benefit rather than an error, enabling Muon with momentum to find stationary points of nonsmooth objectives?}
\end{center}

\subsection{Contributions of this paper}

We answer this question affirmatively.
With momentum and finite Newton--Schulz, Muon finds stationary points of nonsmooth objectives, whereas with the exact-polar update it may fail to converge~\citep{parshakova26muon}.
The reason is that the finite iteration replaces the discontinuous polar map with a Lipschitz map of the singular values, and this smoothing enables the online-to-nonconvex conversion to provide a stationarity guarantee.

The following informal statement summarizes the resulting guarantee.
\begin{theorem}[Informal version of~\cref{thm:ns-complexity}]
\label{thm:informal}
Consider a Lipschitz objective that may be neither smooth nor convex, accessed through an unbiased stochastic gradient oracle with bounded noise. Then, for any radius $\rho>0$ and accuracy $\epsilon>0$,
Muon with momentum and $q=O(\log(1/\epsilon))$ Newton--Schulz steps per round finds a $(\rho,\epsilon)$-stationary point in expectation within $O(\rho^{-1}\epsilon^{-3}+\epsilon^{-2})$ stochastic gradient evaluations.
\end{theorem}

The stationarity criterion is the $(\rho,\epsilon)$-stationarity of~\citet{jiang26adaptive}, stated as~\cref{def:rho-stationarity}, and the leading $\rho^{-1}\epsilon^{-3}$ dependence matches the guarantees established for Pion and Leon under the same criterion~(see~\cref{app:comparison}).
To our knowledge, this is the first stationarity guarantee for nonsmooth nonconvex objectives in which finite Newton--Schulz acts as the smoothing mechanism that enables convergence rather than as an approximation error to be controlled.
The required depth grows only logarithmically in the target accuracy, which is consistent with the empirical observation that a few Newton--Schulz iterations suffice in practice~\citep{jordan2024muon}.
The precise variant we analyze and its remaining differences from deployed Muon are specified in~\cref{sec:setup} and discussed in~\cref{sec:conclusion}.

\paragraph{Technical contributions.}
By the online-to-nonconvex conversion, it suffices to bound the discounted regret that Muon's update rule incurs as an online learner over the operator-norm ball~\citep{cutkosky23optimal,jiang26adaptive}.
We show that this update is the gradient of a smoothed spectral potential evaluated at the momentum, so the learner is a gradient-based prediction algorithm~\citep{abernethy14smoothing,abernethy16perturbation}.
The discounted regret of such an algorithm decomposes into a penalty term, which shrinks as the potential approaches the nuclear norm, and a stability term, which grows with the Lipschitz constant of the induced spectral map.
Unlike the standard decomposition in terms of the cumulative gradients~\citep{jiang26adaptive}, ours is carried out in terms of the momentum, which yields the optimal $O(1/\epsilon^{2})$ dependence for deterministic smooth objectives~(\cref{sec:ns-smooth}).
Our key lemma~(\cref{lem:scalar-ns}) quantifies both terms for finite Newton--Schulz: the penalty decays and the stability grows, both geometrically in $q$.
The depth therefore governs a penalty--stability tradeoff, and balancing the two terms yields sublinear discounted regret at $q=O(\log(1/\epsilon))$ and, through the conversion, the guarantee above~(\cref{sec:ns}).
As $q\to\infty$, the update approaches the exact polar factor, and the stability term grows without bound, reflecting the linear regret that follow-the-leader can suffer.

The Muon learner is also follow-the-regularized-leader (FTRL) on the discounted linear losses, with a spectral regularizer given by the Fenchel conjugate of the smoothed potential~(\cref{thm:general-ftrl}).
The depth controls the amount of regularization: at $q=0$ the regularizer becomes the squared Frobenius regularizer restricted to the operator-norm ball, and as $q\to\infty$ it decays to zero on the ball and the update approaches follow-the-leader.

The analysis is not specific to the Newton--Schulz polynomial.
The regret bound holds for more general spectral maps of the singular values~(\cref{thm:general-discounted-regret}), which is useful to identify sufficient conditions for such maps to yield stationarity guarantees on nonsmooth objectives.
In particular, recent work designs a smooth relaxation of the polar transformation and derives its associated convex regularizer~\citep{mustafi2026move,feoktistov26softsign}, and our general analysis covers such relaxations~(\cref{app:other-spectral-maps}).
In contrast to these studies, we show that the finite Newton--Schulz iteration itself smooths the polar transformation and induces an FTRL regularizer, without introducing a separate relaxation.

\subsection{Related work}

\paragraph{Muon optimizer.}
Muon was proposed by~\citet{jordan2024muon} as an optimizer for the matrix-shaped hidden layers of neural networks.
Spectral update directions were used earlier in preconditioned spectral descent~\citep{carlson15spectral}, and the steepest-descent and duality interpretations of Muon are developed by~\citet{bernstein24old,bernstein25modular}.
Convergence guarantees for Muon with the exact-polar update or an exact linear minimization oracle have been established under smoothness and generalized smoothness~\citep{li25note,kovalev25orthogonalization,pethick25lmo,shen26convergence,chen26spectral,riabinin26gluon,sfyraki26lions}.
Its implicit bias, the denoising role of momentum, and the regimes in which spectral updates outperform Euclidean ones have also been studied~\citep{fan25implicit,li26denoise,davis25spectral,braun26spectral}.
For Muon with the exact-polar update, \citet{parshakova26muon} construct convex Lipschitz objectives on which the iterates fail to converge, which sharpens the question of what changes under finite Newton--Schulz on nonsmooth problems.
Our work takes up this question, analyzing the update that retains both momentum and finite Newton--Schulz on nonsmooth nonconvex objectives.

\paragraph{Newton--Schulz and general spectral maps.}
Iterative orthogonalization by matrix polynomials goes back to~\citet{kovarik70iterative} and~\citet{bjorck71iterative}, and the iterations used in Muon range from empirically tuned polynomials~\citep{jordan2024muon} to optimal polynomial schemes designed for this purpose~\citep{amsel26polar}.
Analyses of Muon that include the finite iteration treat the effect of the finite iteration as an error: for smooth nonconvex objectives~\citep{kim26convergence}, through inexact linear minimization oracles~\citep{shulgin26inexact}, and for Nesterov momentum under heavy-tailed noise~\citep{choudhury26nesterov}.
A few recent works identify benefits of orthogonalization and of its finite approximations in structured settings: on matrix quadratics, finite Newton--Schulz damps directions associated with small singular values near rank deficiency~\citep{shulgin26quadratic}, inexact polar updates can improve reachability on simple strongly convex quadratics~\citep{gonon26insights}, and spectral orthogonalization acts as a preconditioner in matrix factorization and in-context learning models~\citep{ma26preconditioning}.
These benefits are confined to quadratic or otherwise structured objectives.
In contrast, we show that finite Newton--Schulz enables stationarity guarantees for general nonsmooth nonconvex objectives and characterize how the depth trades off polar approximation against the stability of the online updates.
Beyond the exact polar factor, recent work has considered more general spectral maps~\citep{qi26delving,dong26fractional,jiang26clipping,wu26dynmuon}, including smooth relaxations of spectral normalization and orthogonalization~\citep{feoktistov26softsign,mustafi2026move}.
Our theory in~\cref{sec:general} gives sufficient conditions under which such maps enjoy stationarity guarantees for nonsmooth nonconvex objectives.

\paragraph{Online-to-nonconvex conversion.}
The online-to-nonconvex conversion (O2NC) framework was introduced by~\citet{cutkosky23optimal}, who convert online regret guarantees into optimal stationarity guarantees for stochastic nonsmooth nonconvex optimization.
The framework can also exploit additional properties of an objective: it attains the optimal first-order complexity for deterministic smooth objectives~\citep{cutkosky23optimal}, and the best-known complexity when both the gradient and Hessian are Lipschitz~\citep{patitucci26improving}.
The discounted O2NC framework, which is central to our analysis of Muon's momentum, grew out of O2NC variants based on random scaling and model exponential moving averages~\citep{zhang24random,ahn24model}.
The broader online-learning perspective has also been used to understand practical optimizers such as Adam~\citep{ahn24understanding} and schedule-free SGD~\citep{ahn25general}, and to handle heavy-tailed gradient noise~\citep{liu24highprobability,liu26online} and weakly convex optimization~\citep{ji26derandomized}.
Most closely related to our matrix setting, \citet{jiang26adaptive} construct a family of smoothed potentials for the nuclear norm and use the resulting adaptive matrix online learners to derive Pion and Leon, which have a convergence guarantee in nonsmooth optimization.
In contrast, we analyze the finite Newton--Schulz transformation underlying Muon and show that finite depth smooths the exact-polar update by quantifying how the depth trades off polar approximation against the stability of successive updates from the online-to-nonconvex conversion perspective.
A detailed comparison with the closest smooth and nonsmooth guarantees is provided in~\cref{app:comparison}.

\section{Optimization setting and the Muon update}
\label{sec:setup}

This section formalizes the optimization problem, the stationarity criterion, and the Muon update that we analyze.

\paragraph{Notation.}
We work with matrices in $\R^{m\times n}$ with $1\leq m\leq n$ and set
$r\coloneqq\min\{m,n\}=m$.
The case $m>n$ is reduced to this one by transposing every matrix. 
We write $\sigma_1(X)\geq\cdots\geq\sigma_r(X)\geq0$ for the singular values of $X\in\R^{m\times n}$ and set $\sigma(X)\coloneqq(\sigma_1(X),\ldots,\sigma_r(X))$.
We write $\diag(x_1,\ldots,x_r)\in\R^{r\times r}$ for the diagonal matrix with diagonal entries $x_1,\ldots,x_r$.
We write $\inpr{X,Y}\coloneqq \tr(X^\top Y)$ for the Frobenius inner product and define the operator, Frobenius, and nuclear norms by $\lVert X\rVert_{\mathrm{op}}\coloneqq\sigma_1(X)$, $\lVert X\rVert_{\mathrm{F}}\coloneqq \sqrt{\sum_{i=1}^r\sigma_i(X)^2}$, and $\lVert X\rVert_*\coloneqq\sum_{i=1}^r\sigma_i(X)$.
These norms satisfy $\lVert X\rVert_{\mathrm{op}}\leq \lVert X\rVert_{\mathrm{F}}\leq \sqrt r\,\lVert X\rVert_{\mathrm{op}}$ and $\lVert X\rVert_{\mathrm{F}}\leq \lVert X\rVert_*\leq \sqrt r\,\lVert X\rVert_{\mathrm{F}}$.
For a positive integer $k$, let $[k]\coloneqq\{1,\ldots,k\}$.
For a scalar function $h$, we use $\Lip(h)$ to denote its Lipschitz constant.
For a differentiable function $\Phi$, we use $B_\Phi(M'\Vert M)\coloneqq\Phi(M')-\Phi(M)-\inpr{\nabla\Phi(M),M'-M}$ to denote the Bregman divergence from $M$ to $M'$ induced by $\Phi$.

\paragraph{Polar factor and singular-value maps.}
Let $S=U\diag(\sigma_1(S),\ldots,\sigma_r(S))V^\top$ be a thin singular value decomposition.
Define its polar factor by $\polar(S)\coloneqq U\diag(\mathbf{1}\{\sigma_i(S)>0\})V^\top$, where $\mathbf{1}\{\cdot\}$ denotes the indicator function.
The value of $\polar(S)$ is independent of the chosen thin singular value decomposition.
The operator and nuclear norms are dual, so for every $S\in\R^{m\times n}$ and radius $D>0$,
\begin{equation}
  \sup_{\lVert X\rVert_{\mathrm{op}}\leq D}\inpr{S,X}
  =
  D\lVert S\rVert_*,
  \label{eq:operator-nuclear-duality}
\end{equation}
where the supremum is attained by the matrix $D\polar(S)$.
For a scalar function $h\colon[0,\infty)\to\R$ with $h(0)=0$, define the singular-value map
\begin{equation}
  \calH_h(S)
  \coloneqq U\diag\left(h(\sigma_1(S)),\ldots,h(\sigma_r(S))\right)V^\top
  \label{eq:singular-value-map}
\end{equation}
for $S=U\diag(\sigma_1(S),\ldots,\sigma_r(S))V^\top$.
The value of $\calH_h(S)$ is independent of the chosen thin singular value decomposition.
Indeed, a thin singular value decomposition is unique up to a simultaneous orthogonal change of basis in $U$ and $V$ within each group of repeated positive singular values, and this change cancels in~\cref{eq:singular-value-map}.
The singular vectors associated with zero singular values can be chosen independently in $U$ and $V$, but the corresponding terms vanish since $h(0)=0$.

\subsection{Objective and stationarity}
\label{sec:objective}

Let $\mathcal D$ be a probability distribution on a sample space $\mathcal Z$, and let $\ell\colon\R^{m\times n}\times\mathcal Z\to\R$ be a loss function.
We consider the nonconvex matrix optimization problem
\[
  \min_{W\in\R^{m\times n}}\mathcal L(W)
  \coloneqq \E_{\zeta\sim\mathcal D}[\ell(W;\zeta)].
\]
We impose the following conditions on the objective $\mathcal L$ and the stochastic gradient oracle.

\begin{assumption}
\label{ass:oracle}
The function $\mathcal L$ is differentiable\footnote{As discussed in~\citet[Proposition~2 and Corollary~6]{cutkosky23optimal}, for a locally Lipschitz objective that is not differentiable everywhere, we can construct a differentiable surrogate objective $\widehat{\mathcal L}_\delta(W)\coloneqq\E_U[\mathcal L(W+\delta U)]$, where $\delta>0$ and $U$ is uniform on the Frobenius unit ball, and this does not worsen the resulting stationarity guarantees.} and bounded below.
There are constants $\Gamma > 0$ and $\sigma\geq0$ such that, at every query point $W\in\R^{m\times n}$, the oracle uses a random seed independent of the preceding history and returns a stochastic gradient $G$ satisfying
\begin{equation}
  \E[G\mid W]=\nabla\mathcal L(W),
  \qquad
  \E[\lVert G\rVert_{\mathrm{F}}^2\mid W]\leq\Gamma^2,
  \qquad
  \E[\lVert G-\nabla\mathcal L(W)\rVert_{\mathrm{F}}^2\mid W]\leq\sigma^2.
  \label{eq:oracle-bounds}
\end{equation}
\end{assumption}

For the initial point $W_0\in\R^{m\times n}$, let $\Delta_{\mathcal L}\coloneqq\mathcal L(W_0)-\inf_{W\in\R^{m\times n}}\mathcal L(W)<\infty$.
Note that we have $\lVert\nabla\mathcal L(W)\rVert_{\mathrm{F}}\leq\Gamma$ for every $W$ by Jensen's inequality and~\cref{ass:oracle}, and thus $\mathcal L$ is $\Gamma$-Lipschitz with respect to the Frobenius norm.
We do not assume smoothness of $\mathcal L$.
Without loss of generality, we assume $\Gamma \leq \sqrt{r}\,G_{\mathrm{op}}$ when $\lVert G_t\rVert_{\mathrm{op}} \leq G_{\mathrm{op}}$, since this implies $\E[\lVert G\rVert_{\mathrm F}^2 \mid W] \leq rG_{\mathrm{op}}^2$.

To cover objectives without smoothness, we measure progress by the following notion of $(\rho,\epsilon)$-stationarity.
\begin{definition}[{\citealp[Definition~9]{jiang26adaptive}}]
\label{def:rho-stationarity}
For $W\in\R^{m\times n}$ and $\rho>0$, let $\mathcal P(W;\rho)$ be the set of finitely supported probability distributions $p$ on $\R^{m\times n}$ such that $\E_{Y\sim p}[Y]=W$ and $\E_{Y\sim p} \lVert Y-W\rVert_{\mathrm{op}}\leq\rho$.
For $\epsilon\geq0$, a point $W$ is said to be a \emph{$(\rho,\epsilon)$-stationary point} if 
$\lVert \nabla\mathcal L(W)\rVert_*^{[\rho]} \coloneqq \inf_{p\in\mathcal P(W;\rho)} \lVert \E_{Y\sim p} \nabla\mathcal L(Y) \rVert_* \leq\epsilon$.
\end{definition}

The point mass at $W$ belongs to $\mathcal P(W;\rho)$, so $\lVert \nabla\mathcal L(W)\rVert_*^{[\rho]}\leq\lVert \nabla\mathcal L(W)\rVert_*$.
Other distributions in $\mathcal P(W;\rho)$ allow gradients at different points to cancel on average.
For general Lipschitz nonsmooth objectives, such relaxations are unavoidable~\citep{zhang20complexity,kornowski22oracle}.

\subsection{Muon with momentum and finite Newton--Schulz}
\label{sec:muon}

For matrix parameters, Muon maintains gradient momentum and approximately orthogonalizes it with a few Newton--Schulz iterations.
The specific Muon update investigated in this paper is given in~\cref{alg:momentum-muon}.

\begin{algorithm}[t]
\caption{Muon online learner with momentum and finite Newton--Schulz}
\label{alg:momentum-muon}
\KwRequire{momentum $\beta\in(0,1)$, radius $D>0$, operator-norm bound $G_{\mathrm{op}}>0$, depth $q\in\{0,1,\ldots\}$}
Set $M_0\gets0$ and $X_1\gets0$\;
\For{$t=1,2,\ldots$}{
  Receive gradient feedback $G_t$ after playing $X_t$\;
  Update momentum $M_t\gets\beta M_{t-1}+(1-\beta)G_t$\label{algline:momentum}\;
  Normalize momentum by $Y\gets M_t/G_{\mathrm{op}}$\label{algline:normalization}\;
  \For{$j=1,\ldots,q$}{
    $B\gets YY^\top$,\quad
    $P\gets BY$,\quad
    $Q\gets BP$\;
    $Y\gets \frac{15}{8}Y-\frac54P+\frac38Q$\label{algline:ns-step}\;
  }
  Set $X_{t+1}\gets-DY$\label{algline:action}\;
}
\end{algorithm}

The learner in~\cref{alg:momentum-muon} maintains the momentum sequence
\begin{equation}
  M_0\coloneqq 0,
  \qquad
  M_t\coloneqq \beta M_{t-1}+(1-\beta)G_t,
  \label{eq:ema-momentum}
\end{equation}
where $\beta\in(0,1)$ is the momentum parameter, as in Line~\ref{algline:momentum}.
For the Muon learner, we assume that the stochastic gradients returned by the oracle are bounded in operator norm by the input $G_{\mathrm{op}}>0$, that is, $\lVert G_s\rVert_{\mathrm{op}}\leq G_{\mathrm{op}}$ for every $s$.
The normalization in Line~\ref{algline:normalization} then satisfies $\lVert M_t/G_{\mathrm{op}}\rVert_{\mathrm{op}}\leq1$.
Then the learner applies $q$ Newton--Schulz steps in Line~\ref{algline:ns-step}~\citep{kovarik70iterative,bjorck71iterative}.
Finally, the learner sets the next action to $X_{t+1}=-DY$ in Line~\ref{algline:action}.

In general, truncating the Taylor expansion of $\lambda^{-1/2}$
around $\lambda=1$ at order $\kappa\ge1$ gives
\begin{equation}
p_\kappa(\lambda)
\coloneqq
\sum_{s=0}^{\kappa}
c_s(1-\lambda)^s,
\qquad
c_s\coloneqq\frac{(2s)!}{4^s(s!)^2}. \label{eq:ns-taylor}
\end{equation}
Throughout this paper, we focus on the case $\kappa=2$ as computed in Line~\ref{algline:ns-step}, which is the degree-five Newton--Schulz transformation commonly used in Muon.
On the singular values, one step of Line~\ref{algline:ns-step} then acts as
\[
  f(x)\coloneqq x \, p_2(x^2)=\frac{15}{8}x-\frac54x^3+\frac38x^5,
  \qquad
  A\coloneqq f'(0)=\frac{15}{8}.
\]
For $q\in\{0,1,2,\ldots\}$, let $f^{\circ q}$ denote the $q$-fold composition of $f$, with $f^{\circ0}$ the identity, and define
\begin{equation}
  h_q\coloneqq f^{\circ q}\quad\text{on }[0,1],
  \qquad
  h_q(x)\coloneqq1\quad\text{for }x\geq1.
  \label{eq:hq-definition}
\end{equation}
Under the operator-norm bound above, the procedure in Lines~\ref{algline:normalization}--\ref{algline:action} can be written as
\begin{equation}
  X_{t+1}
  =
  -D\calH_{h_q}(M_t/G_{\mathrm{op}}). 
  \label{eq:ns-spectral-map}
\end{equation}
In the discounted online-to-nonconvex conversion framework introduced in~\cref{sec:o2nc}, we use the learner in~\cref{alg:momentum-muon} as the online learner $\mathcal A$ to determine the update direction in the optimizer.

Note that the learner in \Cref{alg:momentum-muon} differs from practical Muon implementations in several respects, including the fixed normalization discussed in~\cref{sec:conclusion}.
Nevertheless, it retains momentum and uses finitely many Newton--Schulz steps, as practical Muon does.

\section{Discounted online-to-nonconvex conversion}
\label{sec:o2nc}

Online-to-nonconvex conversion (O2NC) uses an online learner to select update directions and converts a bound on regret into a stationarity guarantee~\citep{cutkosky23optimal}.
In particular, we use the discounted O2NC framework of~\citet{jiang26adaptive}, since it uses the same discount factor $\beta$ as the momentum update in~\cref{alg:momentum-muon}.

\paragraph{Conversion mechanism.}

At each round $t$, the online learner chooses an increment $X_t$ from the preceding gradient feedback, and the conversion sets $W_t=W_{t-1}+X_t$.
To relate the online loss to $\mathcal L(W_t)-\mathcal L(W_{t-1})$, it queries the oracle at a random point on the segment between these endpoints.
\Cref{alg:generic-o2nc} formalizes the overview in Section~5.1 of~\citet{jiang26adaptive}, using the protocol in their Appendix~G and the output rule in their Proposition~25.

\begin{algorithm}[t]
\caption{Discounted O2NC with a generic online learner}
\label{alg:generic-o2nc}
\KwRequire{horizon $T$, initial point $W_0$, discount $\beta\in(0,1)$, radius $D>0$, online learner $\mathcal A$}
Initialize $\mathcal A$\;
\For{$t=1, 2, \dots T$}{
  Receive $X_t$ satisfying $\lVert X_t\rVert_{\mathrm{op}}\leq D$ from the online learner $\mathcal A$\;
  Set $W_t\gets W_{t-1}+X_t$\;
  Set $\widetilde W_t\gets W_{t-1}+u_tX_t$ for $u_t\sim\operatorname{Unif}([0,1])$ \label{algline:o2nc-query}\;
  Query the stochastic gradient oracle at $\widetilde W_t$ to obtain $G_t$\label{algline:o2nc-gradient} satisfying~\cref{eq:oracle-bounds} with $W=\widetilde W_t$\;
  Construct and feed the linear loss $\ell_t^{[\beta]}(X)\coloneqq \beta^{-t}\inpr{G_t,X}$ to $\mathcal A$\label{algline:o2nc-loss}\;
}
Form $\bar W_t$ as in~\cref{eq:ewa-iterate}, sample $\tau$ from~\cref{eq:output-distribution}, and return $\bar W_\tau$ \label{algline:random-output}\;
\end{algorithm}

The key observation for O2NC is that we can estimate the function change $\mathcal L(W_t)-\mathcal L(W_{t-1})$ by choosing the query point appropriately, even without smoothness of $\mathcal L$.
As in Line~\ref{algline:o2nc-query}, the conversion chooses the query point $\widetilde W_t= W_{t-1}+u_tX_t$ for $u_t\sim\operatorname{Unif}([0,1])$ and observes a conditionally unbiased gradient $G_t$ satisfying~\cref{ass:oracle}, as in Line~\ref{algline:o2nc-gradient}.
With this choice of query point, under~\cref{ass:oracle} the fundamental theorem of calculus along the segment from $W_{t-1}$ to $W_t$ gives
\begin{equation}
  \mathcal L(W_t)-\mathcal L(W_{t-1})
  =
  \int_0^1\inpr{\nabla\mathcal L(W_{t-1}+uX_t),X_t}\,\mathrm{d}u
  =
  \E_{u_t}\left[\inpr{\nabla\mathcal L(\widetilde W_t),X_t}\right]
  .
  \notag
\end{equation}
This implies that $\mathcal L(W_t)-\mathcal L(W_{t-1})= \E[\inpr{G_t,X_t}\mid W_{t-1},X_t]$.
By this identity, controlling the decrease of the objective reduces to online linear optimization with gradient feedback $G_t$, without any smoothness of $\mathcal L$.
The conversion then constructs the loss function $\ell_t^{[\beta]}$ and feeds it to the online learner, as in Line~\ref{algline:o2nc-loss}.

After $T$ rounds, as in Line~\ref{algline:random-output}, the conversion computes the exponentially weighted average (EWA) iterates
\begin{equation}
  \bar W_t
  \coloneqq \frac{1-\beta}{1-\beta^t}
    \sum_{s=1}^t\beta^{t-s}\widetilde W_s,
  \qquad t=1,\ldots,T,
  \label{eq:ewa-iterate}
\end{equation}
and then chooses the randomized output $\bar W_\tau$, where the output time $\tau\in[T]$ is drawn independently according to the distribution given by
\begin{equation}
  \P(\tau=t)
  \coloneqq \begin{cases}
    \dfrac{1-\beta^t}{T}, & 1\leq t<T,\\
    \dfrac{1-\beta^T}{(1-\beta)T}, & t=T.
  \end{cases}
  \label{eq:output-distribution}
\end{equation}

\paragraph{Conversion guarantee.}
The discounted O2NC framework transforms a discounted-regret bound of an online learner into a guarantee for $(\rho,\epsilon)$-stationarity in~\cref{def:rho-stationarity}.
For a fixed terminal time $t$, define the discounted regret of an online learner run over the feasible set $\{X\in\R^{m\times n} \colon \lVert X\rVert_{\mathrm{op}}\leq D\}$ by
\begin{equation}
  \DReg_t(D)
  \coloneqq \max_{\lVert X\rVert_{\mathrm{op}}\leq D}
  \sum_{s=1}^t\beta^{t-s}\inpr{G_s,X_s-X}.
  \label{eq:discounted-regret-definition}
\end{equation}

The following lemma is a corollary of~\citet[Proposition~25]{jiang26adaptive}, and the proof is given in~\cref{app:o2nc-proof}.

\begin{lemma}[{Corollary of \citealp[Proposition~25]{jiang26adaptive}}]
\label{lem:jiang-o2nc}
Under~\cref{ass:oracle}, fix $\rho>0$ and run~\cref{alg:generic-o2nc} with $D=(1-\beta)\rho/(4\beta)$.
Then
\begin{align}
  \E\brk*{\lVert\nabla\mathcal L(\bar W_\tau)\rVert_*^{[\rho]}}
  &\leq
  \frac{4\Delta_{\mathcal L}}{(1-\beta)\rho T}
  +\frac1T\E\left[
      \DReg_T(1)+(1-\beta)\sum_{t=1}^{T-1}\DReg_t(1)
    \right]
  \notag\\
  &\qquad
  +\left(1-\beta+\frac{\beta}{T}\right)
  \frac{\sqrt r\,\sigma}{\sqrt{1-\beta^2}},
  \notag
\end{align}
where $\DReg_t(1) \coloneqq \max_{\lVert U\rVert_{\mathrm{op}}\leq 1} \sum_{s=1}^t\beta^{t-s}\inpr{G_s, X_s / D - U}$ so that $\DReg_t(D)=D\DReg_t(1)$.
\end{lemma}

The only quantities in the bound of~\cref{lem:jiang-o2nc} that depend on the online learner are the expected discounted regrets $\E[\DReg_t(1)]$, which we bound in the subsequent sections.

\paragraph{Exact-polar Muon as follow-the-leader.}

For the gradients $(G_s)_{s=1}^t$ returned at the O2NC query points, we have $M_t=(1-\beta)\beta^t\sum_{s=1}^t\beta^{-s}G_s$.
Hence, Muon with the exact-polar computation, in which the Newton--Schulz iteration is replaced by the exact polar factor, can be written as
\begin{equation}
  X_{t+1}
  =-D\polar(M_t)
  =-D\polar\left(\sum_{s=1}^t\beta^{-s}G_s\right)
  \in\argmin_{\lVert X\rVert_{\mathrm{op}}\leq D} \, \sum_{s=1}^t \inpr*{\beta^{-s}G_s,X},
  \label{eq:exact-polar-ftl}
\end{equation}
where we used $\polar(cS)=\polar(S)$ for $c>0$, and the inclusion follows by applying~\cref{eq:operator-nuclear-duality} to $-\sum_{s=1}^t\beta^{-s}G_s$.
This implies that the Muon update with the exact-polar computation is equivalent to the follow-the-leader (FTL) algorithm over the operator-norm ball for the loss sequence $(\ell_s^{[\beta]})_{s=1}^t$, as noted by~\citet[Section~5.2]{jiang26adaptive}.

Since FTL need not achieve sublinear regret, the discounted O2NC framework alone does not provide a general stationarity guarantee for Muon with the exact-polar update on nonsmooth nonconvex objectives.
This is consistent with the counterexamples of~\citet{parshakova26muon}, who construct nonsmooth convex Lipschitz objectives on which Muon with the exact-polar update and subgradients evaluated at the current iterate fails to converge.
To obtain convergence guarantees for nonsmooth nonconvex optimization, \citet{jiang26adaptive} instead replace FTL with online algorithms that enforce stability explicitly.
In contrast, \cref{sec:ns} shows that the finite Newton--Schulz iteration in~\cref{alg:momentum-muon}, though introduced only to approximate the polar factor numerically, achieves sublinear discounted regret with an appropriate normalization and depth.

\section{Discounted regret for a matrix online learner with a general spectral map}
\label{sec:general}
\begin{figure}[t]
  \centering
    \includegraphics[width=0.37\linewidth]{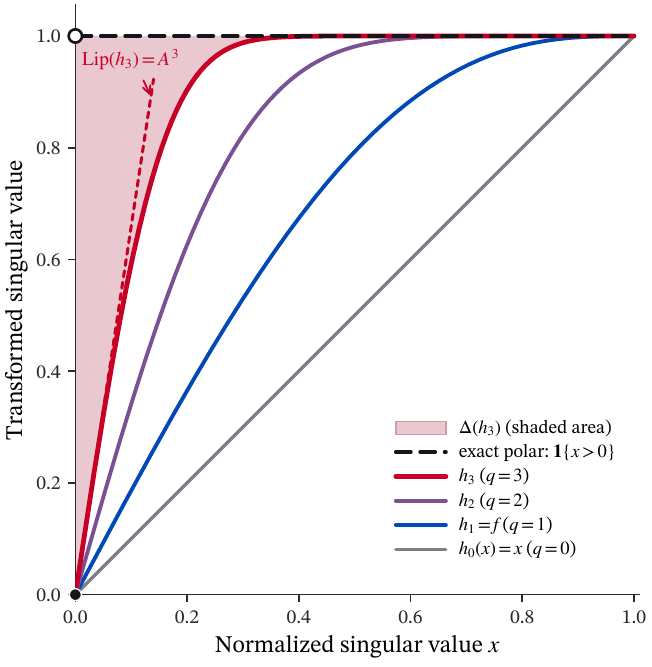}
  \caption{Finite Newton--Schulz smooths the exact polar map.
  For a normalized singular value $x\in[0,1]$, increasing $q$ decreases the gap $\Delta(h_q)$ between the Newton--Schulz map and the exact polar map but increases the Lipschitz constant $\Lip(h_q)=A^q$ of the map.
  The shaded region is $\Delta(h_3)=\int_0^1(1-h_3(x))\,\mathrm{d}x$.}
  \label{fig:scalar-ns-tradeoff}
\end{figure}

The conversion of~\cref{sec:o2nc} reduces nonconvex stationarity to bounding the discounted regret of the online learner over the operator-norm ball.
This section provides that bound for the learner that maintains momentum and applies a general spectral map to it.
The analysis is based on the gradient-based prediction algorithm~(GBPA) framework~\citep{abernethy14smoothing,abernethy16perturbation}, in which the learner plays the gradient of a differentiable potential function of the aggregated gradient feedback.
In our setting with the loss function $\ell_t^{[\beta]}$, this aggregate corresponds to the momentum $M_{t-1}$, so the learner plays $X_t=-D\nabla\widetilde\Phi_t(M_{t-1})$ for some potential function $\widetilde\Phi_t$, and we consider the regret decomposition in terms of the momentum sequence.
\Cref{sec:ns} specializes the result to Muon with finite Newton--Schulz in~\Cref{alg:momentum-muon}.

We consider a spectral map $h\colon[0,\infty)\to[0,1]$ satisfying the following conditions.
\begin{assumption}
\label{ass:general-spectral-link}
The function $h\colon[0,\infty)\to[0,1]$ is continuous, nondecreasing, and satisfies 
\begin{equation}
  h(0)=0,
  \qquad
  \Lip(h)<\infty,
  \qquad
  \Delta(h)\coloneqq \int_0^\infty (1 - h(u)) \,\mathrm{d}u < \infty.
  \notag
\end{equation}
\end{assumption}
Here the approximation error $\Delta(h)$ measures the gap between $h$ and the exact polar map, which sends every positive singular value to one, and a smaller $\Delta(h)$ forces a larger $\Lip(h)$.
The function $h_q$ in~\cref{eq:hq-definition} satisfies~\cref{ass:general-spectral-link}, as verified in~\cref{sec:ns-regret}, and \Cref{fig:scalar-ns-tradeoff} illustrates this tradeoff for $h = h_q$.

\subsection{Online learner and discounted-regret bound}

Using the momentum~\cref{eq:ema-momentum} and a normalization scale $C_t>0$, we consider the online learner that takes the action
\begin{equation}
  X_t^h
  \coloneqq-D\calH_h\left(\frac{M_{t-1}}{C_t}\right),
  \label{eq:general-spectral-action}
\end{equation}
which satisfies $\lVert X_t^h\rVert_{\mathrm{op}}\leq D$ because $h(\cdot)\in[0,1]$.
For $h=h_q$ and $C_t = G_{\mathrm{op}}$, the action~\cref{eq:general-spectral-action} is the finite Newton--Schulz update in~\cref{eq:ns-spectral-map} provided that $\lVert M_t \rVert_{\mathrm{op}} \leq G_{\mathrm{op}}$ for all $t$.
Replacing $\calH_h(\cdot)$ with $\polar(\cdot)$ recovers Muon with the exact-polar computation, which corresponds to the update~\cref{eq:exact-polar-ftl}.
The potential corresponding to exact-polar Muon is the base potential $\Phi(M)\coloneqq\lVert M\rVert_*$, one of whose subgradients is the exact polar factor.
The update~\cref{eq:general-spectral-action} replaces this base potential $\Phi$ by the smoothed potential given by
\begin{equation}
  \widetilde\Phi_{h,C}(M)
  \coloneqq C\sum_{i=1}^r
  \phi_h\left(\frac{\sigma_i(M)}{C}\right)
  \qquad
  \text{for}\quad
  \phi_h(x)\coloneqq\int_0^x h(u)\,\mathrm{d}u
  .
  \label{eq:general-spectral-potential}
\end{equation}

The regret bound under the update~\cref{eq:general-spectral-action} depends on $h$ through two scalar quantities: the Lipschitz constant $\Lip(h)$ and the approximation error $\Delta(h)$.

\begin{theorem}
\label{thm:general-discounted-regret}
Suppose~\cref{ass:general-spectral-link} holds.
For any $t\geq1$, let $(C_s)_{s=1}^t$ be positive reals satisfying $C_{s+1}\geq\beta C_s$ for $s \in [t-1]$.
Then the online learner~\cref{eq:general-spectral-action} achieves
\begin{equation}
  \DReg_t(D)
  \leq D\left(
    \frac{rC_t\Delta(h)}{1-\beta}
    +
    \frac{1}{1-\beta}
    \sum_{s=1}^t\beta^{t-s}B_{\widetilde\Phi_{h,C_s}}(M_s\Vert M_{s-1})
  \right).
  \label{eq:bregman-innovation}
\end{equation}
If, in addition, \cref{ass:oracle} holds and $C_s=C$ for all $s\in[t]$, then
\begin{equation}
  \E[\DReg_t(D)]
  \leq D\left(
    \frac{r C \Delta(h)}{1-\beta}
    +\frac{2 \Gamma^2 \Lip(h)}{C}
  \right).
  \label{eq:expected-discounted-regret}
\end{equation}
\end{theorem}
The coarser bound~\cref{eq:expected-discounted-regret} is sufficient for the guarantee for nonsmooth nonconvex objectives in~\cref{thm:ns-complexity}, while the guarantee for smooth objectives in~\cref{thm:smooth-stationarity} relies on the bound depending on the Bregman divergence~\cref{eq:bregman-innovation} to obtain its optimal $O(1/\epsilon^2)$ dependence in the deterministic case.
Notably, the bound~\cref{eq:expected-discounted-regret} has the same structure as the well-known penalty--stability tradeoff in the analysis of FTRL and online mirror descent~\citep[\textit{e.g.,}][]{shalevshwartz12online,mcmahan17survey,zimmert21tsallis,tsuchiya23stability}, with a penalty term proportional to $\Delta(h)$ and a stability term proportional to $\Lip(h)$.

\subsection{Proof of~\cref{thm:general-discounted-regret}}

Here we provide the proof of~\cref{thm:general-discounted-regret}.

\subsubsection{The representation as a gradient-based prediction algorithm (GBPA)}
The action~\cref{eq:general-spectral-action} normalizes the momentum $M_{t-1}$ by the scale $C_t$.
We first show that it is a GBPA update $X_t=-D\nabla\widetilde\Phi_t(M_{t-1})$ for the smoothed potential $\widetilde\Phi_t=\widetilde\Phi_{h,C_t}$ of~\cref{eq:general-spectral-potential}.
The following lemma gives its gradient.
\begin{lemma}
\label{lem:spectral-calculus}
Under~\cref{ass:general-spectral-link}, for every $C>0$ the function $\widetilde\Phi_{h,C}$ is convex and continuously differentiable on all of $\R^{m\times n}$, and its gradient is
\begin{equation}
  \nabla\widetilde\Phi_{h,C}(M) = \calH_h(M/C).
  \label{eq:spectral-gradient}
\end{equation}
\end{lemma}

\begin{proof}
Write $\widetilde\Phi_{h,C}(M)=\sum_{i=1}^r\phi(\sigma_i(M))$ with $\phi(x)\coloneqq C\phi_h(x/C)$.
Since $h$ is continuous, nonnegative, and nondecreasing with $h(0)=0$, the function $\phi$ is convex, nondecreasing, and differentiable, and its derivative satisfies $\phi'(x)=h(x/C)$ and $\phi'(0)=0$.
Hence, by the convexity of separable singular-value functions~(\cref{lem:aux-sv-convexity}), the function $\widetilde\Phi_{h,C}$ is convex on $\R^{m\times n}$, and by the standard singular-value gradient formula~(\cref{lem:aux-sv-gradient}), its gradient is $\nabla\widetilde\Phi_{h,C}(M)=U\diag\left(h(\sigma_i(M)/C)\right)V^\top=\calH_h(M/C)$, which completes the proof.
\end{proof}

By~\cref{eq:spectral-gradient}, the action~\cref{eq:general-spectral-action} can be written as
\begin{equation}
  X_t^h
  =-D\calH_h\left(\frac{M_{t-1}}{C_t}\right)
  =-D\nabla\widetilde\Phi_{h,C_t}(M_{t-1}),
  \label{eq:general-gbpa-action}
\end{equation}
which is the GBPA update with the time-varying potential $\widetilde\Phi_t=\widetilde\Phi_{h,C_t}$ evaluated at the momentum $M_{t-1}$.
A GBPA also has a dual description: with a convex potential, it is follow-the-regularized-leader (FTRL) with the regularizer given by the Fenchel conjugate of the potential~\citep{abernethy14smoothing}, and~\cref{sec:ftrl} develops this correspondence for the spectral potentials $\widetilde\Phi_{h,C}$.

\subsubsection{Regret decomposition based on momentum for GBPA}

We now derive the discounted-regret bound of the GBPA update~\cref{eq:general-gbpa-action}.
Unlike the standard GBPA decomposition in terms of the cumulative gradients~\citep{abernethy14smoothing,abernethy16perturbation}, employed by~\citet{jiang26adaptive}, the following lemma decomposes the discounted regret directly in terms of the momentum sequence.

\begin{lemma}
\label{lem:momentum-gbpa}
Under~\cref{ass:general-spectral-link}, the discounted regret for the online learner~\cref{eq:general-gbpa-action} is decomposed as
\begin{align}
  \DReg_t(D)
  =&\frac{D}{1-\beta}\biggl(
    \Phi(M_t)-\widetilde\Phi_{h,C_t}(M_t)
    +\sum_{s=1}^t\beta^{t-s}B_{\widetilde\Phi_{h,C_s}}(M_s\Vert M_{s-1})
    +\sum_{s=1}^{t-1}\beta^{t-s-1}V_s
  \biggr),
  \label{eq:momentum-gbpa}
\end{align}
where
$V_s\coloneqq\widetilde\Phi_{h,C_{s+1}}(M_s)-\beta\widetilde\Phi_{h,C_s}(M_s)-(1-\beta)\inpr{M_s,\calH_h(M_s/C_{s+1})}$.
\end{lemma}

On the right-hand side of~\cref{eq:momentum-gbpa}, the first term is a penalty term, the gap between the base potential $\Phi$ and the smoothed potential $\widetilde\Phi_{h,C_t}$, the second term is a stability term, the accumulated Bregman divergences along the momentum trajectory, and the last term is the variation of the potentials across rounds.
We show that the penalty term is controlled by the approximation error $\Delta(h)$, the stability term by the Lipschitz constant $\Lip(h)$, and the third term is nonpositive under the scale condition $C_{s+1}\geq\beta C_s$.

\begin{proof}
By the definition of the discounted regret in~\cref{eq:discounted-regret-definition} and $M_t=(1-\beta)\sum_{s=1}^t\beta^{t-s}G_s$, we have
\begin{align*}
  \DReg_t(D)
  &=\sum_{s=1}^t\beta^{t-s}\inpr{G_s,X_s}
  +\max_{\lVert X\rVert_{\mathrm{op}}\leq D}\inpr*{-\frac{M_t}{1-\beta},X}
  =\sum_{s=1}^t\beta^{t-s}\inpr{G_s,X_s}
  + \frac{D}{1 - \beta}\lVert M_t\rVert_*,
\end{align*}
where we used the fact that the maximum in the operator--nuclear duality~\cref{eq:operator-nuclear-duality} is attained at $X=-D\polar(M_t)$.
Substituting the actions~\cref{eq:general-gbpa-action} into the last equality and multiplying by $1-\beta$ gives
\begin{equation}
  (1-\beta)\DReg_t(D)
  =-D\sum_{s=1}^t(1-\beta)\beta^{t-s}
  \inpr{G_s,\nabla\widetilde\Phi_{h,C_s}(M_{s-1})} + D\lVert M_t\rVert_*.
  \label{eq:momentum-regret-start}
\end{equation}

We evaluate the first term in~\cref{eq:momentum-regret-start}.
By using $(1-\beta)G_s=M_s-M_{s-1}+(1-\beta)M_{s-1}$ and the definition of the Bregman divergence, we have
\begin{align}
  &(1-\beta)\inpr{G_s,\nabla\widetilde\Phi_{h,C_s}(M_{s-1})} \notag \\
  &=\inpr{M_s-M_{s-1}, \nabla\widetilde\Phi_{h,C_s}(M_{s-1})}
  +
  (1-\beta) \inpr{M_{s-1},\nabla\widetilde\Phi_{h,C_s}(M_{s-1})} \notag \\
  &=\widetilde\Phi_{h,C_s}(M_s)
  -\widetilde\Phi_{h,C_s}(M_{s-1})
  -B_{\widetilde\Phi_{h,C_s}}(M_s\Vert M_{s-1})
  +(1-\beta)\inpr{M_{s-1},\nabla\widetilde\Phi_{h,C_s}(M_{s-1})}. \label{eq:momentum-regret-1}
\end{align}
Now, by $M_0=0$ and $\widetilde\Phi_{h,C_1}(0)=0$, we have
\begin{equation}
  \!\!
  \sum_{s=1}^t\beta^{t-s}
  \left(\widetilde\Phi_{h,C_s}(M_s)-\widetilde\Phi_{h,C_s}(M_{s-1})\right)
  =
  \widetilde\Phi_{h,C_t}(M_t)
  +
  \sum_{s=1}^{t-1}\beta^{t-s-1}
  \left(\beta\widetilde\Phi_{h,C_s}(M_s)-\widetilde\Phi_{h,C_{s+1}}(M_s)\right).
  \label{eq:momentum-regret-2}
\end{equation}
We also have
\begin{equation}
  \sum_{s=1}^t\beta^{t-s}\inpr{M_{s-1},\nabla\widetilde\Phi_{h,C_s}(M_{s-1})}
  =\sum_{s=1}^{t-1}\beta^{t-s-1}\inpr{M_s,\nabla\widetilde\Phi_{h,C_{s+1}}(M_{s})},
  \label{eq:momentum-regret-3}
\end{equation}
where we used $M_0=0$.
Substituting~\cref{eq:momentum-regret-1,eq:momentum-regret-2,eq:momentum-regret-3} into the right-hand side of~\cref{eq:momentum-regret-start} and dividing by $1-\beta$ completes the proof.
\end{proof}

\subsubsection{Bounding the three terms}
The following lemma bounds the penalty term.
\begin{lemma}
\label{lem:penalty-term}
Under~\cref{ass:general-spectral-link}, for every $M\in\R^{m\times n}$ and $C>0$,
\[
  0\leq\Phi(M)-\widetilde\Phi_{h,C}(M)
  \leq rC\Delta(h).
\]
\end{lemma}

\begin{proof}
By~\cref{eq:general-spectral-potential} and the definition of $\Delta(h)$ in~\cref{ass:general-spectral-link},
\[
  \Phi(M)-\widetilde\Phi_{h,C}(M)
  =
  C \sum_{i=1}^r \left( \frac{\sigma_i(M)}{C} - \phi_h\left( \frac{\sigma_i(M)}{C} \right)\right)
  =
  C\sum_{i=1}^r\int_0^{\sigma_i(M)/C}(1-h(u))\,\mathrm{d}u
  \leq
  r C \Delta(h).
\]
This completes the proof.
\end{proof}

The stability term is controlled by the following smoothness estimate for the potential.
\begin{lemma}
\label{lem:bregman-upper}
Under~\cref{ass:general-spectral-link}, for every $C>0$ the gradient of $\widetilde\Phi_{h,C}$ is $(\Lip(h)/C)$-Lipschitz in the Frobenius norm.
Consequently,
\begin{equation}
  B_{\widetilde\Phi_{h,C}}(M'\Vert M)
  \leq\frac{\Lip(h)}{2C}\lVert M'-M\rVert_{\mathrm F}^2.
  \notag
\end{equation}
\end{lemma}
\begin{proof}
By the Lipschitzness of singular-value maps~(\cref{lem:aux-sv-lipschitz}) applied to $h$ satisfying~\cref{ass:general-spectral-link},
for any $M,M'\in\R^{m\times n}$,
\begin{equation}
 \lVert\calH_h(M/C)-\calH_h(M'/C)\rVert_{\mathrm F}
 \leq\frac{\Lip(h)}{C}\lVert M-M'\rVert_{\mathrm F},
  \label{eq:spectral-map-lipschitz}
\end{equation}
and thus $\nabla\widetilde\Phi_{h,C}=\calH_h(\cdot/C)$ is $(\Lip(h)/C)$-Lipschitz.
By applying the fundamental theorem of calculus to $u\mapsto\widetilde\Phi_{h,C}(M+u(M'-M))$ on $[0,1]$ and then using the Cauchy--Schwarz inequality and~\Cref{eq:spectral-map-lipschitz}, for any $M,M'\in\R^{m\times n}$ we have
\begin{align}
  B_{\widetilde\Phi_{h,C}}(M'\Vert M)
  &=
  \int_0^1\inpr*{\nabla\widetilde\Phi_{h,C}(M+u(M'-M))-\nabla\widetilde\Phi_{h,C}(M),\,M'-M}\,\mathrm{d}u \notag \\
  &\leq
  \int_0^1
  \| \nabla\widetilde\Phi_{h,C}(M+u(M'-M))-\nabla\widetilde\Phi_{h,C}(M)\|_{\mathrm F} \| M'-M \|_{\mathrm F} \,\mathrm{d}u
  \leq
  \frac{\Lip(h)}{2 C} \| M'-M \|_{\mathrm F}^2
  ,
  \notag
\end{align}
which completes the proof.
\end{proof}

Finally, the last term in~\cref{eq:momentum-gbpa} is nonpositive under the scale condition due to the convexity of $\phi_h$ and $\phi_h(0) = 0$.

\begin{lemma}
\label{lem:scale-condition}
Let $0<\beta<1$ and $C,C'>0$ with $C'\geq\beta C$.
Under~\cref{ass:general-spectral-link}, for every $M\in\R^{m\times n}$,
\begin{equation}
  \widetilde\Phi_{h,C'}(M)
  -\beta\widetilde\Phi_{h,C}(M)
  -(1-\beta)\inpr{M,\nabla\widetilde\Phi_{h,C'}(M)}
  \leq0.
  \notag
\end{equation}
Consequently, $V_s\leq0$ whenever $C_{s+1}\geq\beta C_s$.
\end{lemma}
\begin{proof}
The left-hand side of the desired inequality decomposes over the singular values of $M$, so it suffices to show, for each singular value $\lambda$ of $M$, that
\[
  C'\phi_h(\lambda/C')-\beta C\phi_h(\lambda/C)-(1-\beta)\lambda\,h(\lambda/C')\leq0.
\]
Write $x\coloneqq\lambda/C'$ and $\theta\coloneqq\beta C/C'\in(0,1]$, so that $\lambda/C=\beta x/\theta$. 
Then, dividing the above inequality by $C'$, we can see that it suffices to show that
\[
  F(\theta)
  \coloneqq
  \phi_h(x)- \theta \phi_h(\beta x/\theta) - (1-\beta) x h(x)\leq0.
\]
This directly follows from the convexity of $\phi_h$ and $\phi_h(0) = 0$.
Indeed, these two properties imply $\phi_h(\beta x) = \phi_h(\theta \cdot (\beta x / \theta) + (1 - \theta) 0) \leq \theta \phi_h(\beta x / \theta)$.
Again by the convexity of $\phi_h$ and $\phi_h' = h$, we have
$\phi_h(\beta x) \geq \phi_h(x) + (\beta x - x) \phi_h'(x) = \phi_h(x) + (\beta x - x) h(x)$. Combining these two inequalities, we obtain
\begin{equation*}
  F(\theta) =
  \phi_h(x)- \theta \phi_h(\beta x/\theta) - (1-\beta) x h(x)
  \leq
  \phi_h(x) - \phi_h(\beta x) - (1 - \beta) x h(x) \leq 0
  ,
\end{equation*}
which completes the proof.
\end{proof}

We are now ready to prove~\Cref{thm:general-discounted-regret}.
\begin{proof}[Proof of~\Cref{thm:general-discounted-regret}]
We upper bound the right-hand side of~\cref{eq:momentum-gbpa} in~\cref{lem:momentum-gbpa}.
Since the first term of~\cref{eq:momentum-gbpa} is at most $DrC_t\Delta(h)/(1-\beta)$ by~\cref{lem:penalty-term} and $V_s\leq0$ by~\cref{lem:scale-condition}, we obtain~\cref{eq:bregman-innovation}.
For the second claim, let $C_s=C$ for all $s\in[t]$ and take expectations under~\cref{ass:oracle}.
Bounding each Bregman divergence in~\cref{eq:bregman-innovation} by~\cref{lem:bregman-upper}, we have
\begin{equation}
  \E\brk*{
  \sum_{s=1}^t\beta^{t-s}B_{\widetilde\Phi_{h,C_s}}(M_s\Vert M_{s-1})
  }
  \leq
  \frac{\Lip(h)}{2C} \,
  \E\brk*{
  \sum_{s=1}^t \beta^{t-s} \lVert M_s-M_{s-1} \rVert_{\mathrm F}^2
  }
  \leq
  \frac{2 (1 - \beta) \Gamma^2 \Lip(h)}{C}.
  \notag
\end{equation}
Here, the last inequality follows from $M_s-M_{s-1}=(1-\beta)(G_s-M_{s-1})$, the bound $\E[\lVert G_s-M_{s-1}\rVert_{\mathrm F}^2]\leq4\Gamma^2$, which follows from~\cref{ass:oracle} and Jensen's inequality, and $(1-\beta)\sum_{s=1}^t\beta^{t-s}\leq1$.
Substituting this bound into~\cref{eq:bregman-innovation} gives~\cref{eq:expected-discounted-regret}, which completes the proof.
\end{proof}

\section{Regret and stationarity guarantees for Muon with finite Newton--Schulz}
\label{sec:ns}

This section specializes the general theory of~\cref{sec:general} to the finite Newton--Schulz map and to~\cref{alg:momentum-muon}.
The proofs of this section are deferred to~\cref{app:ns-proofs,app:ftrl}.

\subsection{Discounted regret for Muon with finite Newton--Schulz}
\label{sec:ns-regret}

Recall the function $h_q$ from~\cref{eq:hq-definition} and the constant $A=15/8$.
Throughout this section, we use the constant normalization $C_t=G_{\mathrm{op}}$ and assume the operator-norm bound $\lVert G_s\rVert_{\mathrm{op}}\leq G_{\mathrm{op}}$ of~\cref{sec:muon}.
First, we verify that $h_q$ satisfies~\cref{ass:general-spectral-link}.
By~\cref{eq:hq-definition}, $h_q$ is the $q$-fold composition of $f$ on $[0,1]$, extended by the constant $1$ on $[1,\infty)$.
Since $f$ is a polynomial that is nondecreasing on $[0,1]$ with $f(0)=0$ and $f(1)=1$, $h_q\colon[0,\infty)\to[0,1]$ is continuous and nondecreasing, satisfies $h_q(0)=0$ and $h_q(x)=1$ for $x\geq1$, and has a finite Lipschitz constant.

Since the momentum update~\cref{eq:ema-momentum} and the operator-norm bound $\lVert G_s\rVert_{\mathrm{op}}\leq G_{\mathrm{op}}$ give $\lVert M_t\rVert_{\mathrm{op}}\leq(1-\beta^t)G_{\mathrm{op}}\leq G_{\mathrm{op}}$ for every $t$, all singular values of $M_t/G_{\mathrm{op}}$ lie in $[0,1]$, and the finite Newton--Schulz iteration coincides with the singular-value map $\calH_{h_q}(M_t/G_{\mathrm{op}})$.
Thus~\cref{alg:momentum-muon} is exactly the online learner of~\cref{eq:general-spectral-action} with normalization $C_t=G_{\mathrm{op}}$, which satisfies the scale condition $C_{t+1}=G_{\mathrm{op}}\geq\beta C_t$ and keeps the action feasible, $\lVert X_{t+1}\rVert_{\mathrm{op}}\leq D$.
Specializing~\cref{thm:general-discounted-regret} to $h=h_q$ and choosing the depth $q$ via our key lemma~(\cref{lem:scalar-ns} below), we obtain the following discounted-regret bound.
\begin{theorem}
\label{thm:fixed-depth-discounted-regret}
Suppose~\cref{ass:oracle} holds, $\Gamma>0$, and $\lVert G_s\rVert_{\mathrm{op}}\leq G_{\mathrm{op}}$ almost surely for every $s$.
If $G_{\mathrm{op}}\sqrt r/(\Gamma\sqrt{1-\beta})\geq1$, then for every $t\geq1$, the learner of~\cref{alg:momentum-muon} achieves
\begin{equation}
  \E[\DReg_t(D)]
  \leq (1+2A)D\Gamma\sqrt{\frac{r}{1-\beta}}
  \quad\text{with depth}\quad
  q=\left\lceil\log_A\frac{G_{\mathrm{op}}\sqrt r}{\Gamma\sqrt{1-\beta}}\right\rceil.
  \label{eq:fixed-depth-discounted-regret}
\end{equation}
\end{theorem}

The bound of~\cref{thm:fixed-depth-discounted-regret} is obtained from~\cref{eq:expected-discounted-regret} in~\cref{thm:general-discounted-regret} with $h=h_q$, by balancing a penalty term controlled by the approximation error $\Delta(h_q)$ against a stability term controlled by the Lipschitz constant $\Lip(h_q)$.
The following lemma shows that $\Lip(h_q)$ grows exactly as $A^q$, while $\Delta(h_q)$ remains within constant factors of $A^{-q}$.

\begin{lemma}
\label{lem:scalar-ns}
For every integer $q\geq0$, the function $h_q$ in~\cref{eq:hq-definition} satisfies
\begin{equation}
  \Lip(h_q)=A^q,
  \qquad
  \frac{1}{2A^q}
  \leq \Delta(h_q)
  \leq \frac{1}{A^q+1}
  \leq \frac{1}{A^q},
  \qquad
  A = \frac{15}{8}.
  \label{eq:delta-q-bounds}
\end{equation}
\end{lemma}

\begin{proof}[Proof sketch of~\cref{thm:fixed-depth-discounted-regret}]
Substituting $\Lip(h_q)=A^q$ and $\Delta(h_q)\leq A^{-q}$ from~\cref{lem:scalar-ns} into~\cref{eq:expected-discounted-regret} in~\cref{thm:general-discounted-regret} with $C=G_{\mathrm{op}}$ gives
\[
  \E[\DReg_t(D)]
  \leq D\left(\frac{rG_{\mathrm{op}}}{(1-\beta)A^{q}}+\frac{2A^{q}\Gamma^{2}}{G_{\mathrm{op}}}\right).
\]
Balancing the two terms over the integer depth $q$ gives~\cref{eq:fixed-depth-discounted-regret}.
\end{proof}
This tradeoff explains the benefit of a finite Newton--Schulz iteration over exact orthogonalization.
As $q\to\infty$, $h_q$ approaches the exact polar map, and the update approaches the exact polar factor, the ideal Muon update.
In this limit, $\Delta(h_q)\to0$, but $\Lip(h_q)\to\infty$, which reflects the discontinuity of the polar map, and the stability term is no longer controlled.
\Cref{lem:general-scalar-ns} shows that the same tradeoff holds at any finite Taylor order and that our analysis therefore extends beyond the degree-five Newton--Schulz polynomial investigated here.
Moreover, the same penalty--stability tradeoff is unavoidable for general spectral maps, as every $h$ satisfying~\cref{ass:general-spectral-link} obeys $\Delta(h)\geq1/\left(2\Lip(h)\right)$.
By~\cref{lem:scalar-ns}, the finite Newton--Schulz map satisfies $\Delta(h_q)\leq1/\Lip(h_q)$ and hence attains this lower bound up to a factor of two.
\Cref{app:other-spectral-maps} proves this lower bound and instantiates our guarantees for a recently proposed smooth relaxation of the polar map~\citep{mustafi2026move,feoktistov26softsign}.
\Cref{app:numerical-experiments} illustrates this depth tradeoff numerically on a synthetic nonsmooth objective.

\subsection{Sample complexity in nonsmooth nonconvex optimization}
\label{sec:ns-convergence}

Substituting the discounted-regret bound in~\Cref{thm:fixed-depth-discounted-regret} into the online-to-nonconvex conversion framework of~\cref{lem:jiang-o2nc}, we can obtain the following stationarity guarantee for nonsmooth nonconvex optimization.

\begin{theorem}
\label{thm:ns-complexity}
Suppose~\cref{ass:oracle} holds, $\Gamma>0$, and $\lVert G_s\rVert_{\mathrm{op}}\leq G_{\mathrm{op}}$ almost surely for every~$s$, fix $\rho>0$ and $\epsilon>0$, and let $B_{\mathrm{NS}}\coloneqq\sqrt r((1+2A)\Gamma+\sigma)$.
Running~\cref{alg:generic-o2nc} using \cref{alg:momentum-muon} as its online learner $\mathcal A$ with $1-\beta=\min\{1/9,(\epsilon/3B_{\mathrm{NS}})^2\}$, $D=(1-\beta)\rho/(4\beta)$, and depth $q=\left\lceil\log_A\max\left\{1,\ 3B_{\mathrm{NS}}G_{\mathrm{op}}\sqrt r / (\Gamma\epsilon)\right\}\right\rceil$ gives
\begin{equation}
  \E\brk*{\lVert\nabla\mathcal L(\bar W_\tau)\rVert_*^{[\rho]}}
  \leq\frac{\epsilon}{3} + O\left(
    \frac{\Delta_{\mathcal L}\,r(\Gamma+\sigma)^2}{\rho\epsilon^2 T}
    +\frac{r(\Gamma+\sigma)^2}{\epsilon T}
  \right),
  \label{eq:stationarity-master-bound}
\end{equation}
so $\bar W_\tau$ is a $(\rho,\epsilon)$-stationary point in expectation with gradient-oracle complexity
\begin{equation}
  T=O\left(
    \max\left\{
      \frac{\Delta_{\mathcal L}\,r(\Gamma+\sigma)^2}{\rho\epsilon^3},\
      \frac{r(\Gamma+\sigma)^2}{\epsilon^2}
    \right\}
  \right).
  \label{eq:iteration-complexity}
\end{equation}
\end{theorem}

The leading $\rho^{-1}\epsilon^{-3}$ term matches the bounds for Pion and Leon~\citep{jiang26adaptive}, and \cref{app:comparison} gives a detailed comparison.

\subsection{Sample complexity in smooth nonconvex optimization}
\label{sec:ns-smooth}

Using the same discounted O2NC reduction, we also obtain a guarantee for smooth objectives.
In addition to~\cref{ass:oracle} and the operator-norm bound $\lVert G_s\rVert_{\mathrm{op}}\leq G_{\mathrm{op}}$ of~\cref{sec:muon}, throughout this subsection we assume the following smoothness of $\mathcal L$ with respect to the operator norm.

\begin{assumption}
\label{ass:smooth}
There is a constant $L_{\mathrm{op}}>0$ such that
\begin{equation}
  \lVert\nabla\mathcal L(W)-\nabla\mathcal L(W')\rVert_*
  \leq L_{\mathrm{op}}\lVert W-W'\rVert_{\mathrm{op}}
  \qquad\text{for all }W,W'\in\R^{m\times n}.
  \label{eq:op-nuclear-smoothness}
\end{equation}
\end{assumption}

For smooth objectives we use the bound~\cref{eq:bregman-innovation} of~\cref{thm:general-discounted-regret}, stated in terms of the Bregman divergences.
Under~\cref{ass:smooth}, and for $h=h_q$ and $C_s\equiv G_{\mathrm{op}}$, each Bregman divergence $B_{\widetilde\Phi_{h_q,G_{\mathrm{op}}}}(M_s\Vert M_{s-1})$ is bounded more sharply, as shown in~\cref{app:smooth}.
The resulting bound on $\lVert\nabla\mathcal L(W)\rVert_*^{[\rho]}$ is converted to a bound on $\lVert\nabla\mathcal L(W)\rVert_*$ by the following lemma, whose proof is given in~\cref{app:smooth}.

\begin{lemma}
\label{lem:smooth-bridge}
Under~\cref{ass:smooth}, for every $W\in\R^{m\times n}$,
$\lVert\nabla\mathcal L(W)\rVert_*
\leq\lVert\nabla\mathcal L(W)\rVert_*^{[\rho]}+L_{\mathrm{op}}\rho$.
\end{lemma}

Thus, the discounted O2NC conversion of~\cref{lem:jiang-o2nc} gives the following sample complexity guarantee.
\begin{theorem}
\label{thm:smooth-stationarity}
Suppose~\cref{ass:oracle} and~\cref{ass:smooth} hold, $\lVert G_s\rVert_{\mathrm{op}}\leq G_{\mathrm{op}}$ almost surely for every $s$, and fix $\epsilon>0$.
Running~\cref{alg:generic-o2nc} using \cref{alg:momentum-muon} as its online learner $\mathcal A$ with $1-\beta=\min\{1/2,\epsilon^2/(144r\sigma^2)\}$, $D=(1-\beta)\epsilon/(48\beta L_{\mathrm{op}})$, and depth $q=\left\lceil\log_A\max\left\{1,12rG_{\mathrm{op}}/
\epsilon\right\}\right\rceil$ gives
\begin{equation}
  \E\brk*{\lVert\nabla\mathcal L(\bar W_\tau)\rVert_*}
  \leq\frac{2}{3}\epsilon+O\left(
    \frac{\epsilon}{T}
    +\frac{L_{\mathrm{op}}\Delta_{\mathcal L}}{\epsilon T}
    +\frac{r\sigma^2}{\epsilon T}
    +\frac{r\sigma^2L_{\mathrm{op}}\Delta_{\mathcal L}}{\epsilon^3T}
  \right),
  \label{eq:smooth-stationarity}
\end{equation}
so $\bar W_\tau$ satisfies $\E\brk*{\lVert\nabla\mathcal L(\bar W_\tau)\rVert_*}\leq\epsilon$ with gradient-oracle complexity
\begin{equation}
  T
  =O\left(
    \max\left\{
      \frac{L_{\mathrm{op}}\Delta_{\mathcal L}}{\epsilon^2},\
      \frac{r\sigma^2}{\epsilon^2},\
      \frac{r\sigma^2 L_{\mathrm{op}}\Delta_{\mathcal L}}{\epsilon^4}
    \right\}
  \right).
  \label{eq:smooth-complexity}
\end{equation}
\end{theorem}
The $\epsilon^{-2}$ dependence in the deterministic case matches the classical lower bound for smooth nonconvex optimization~\citep{carmon20lower}. Under the standard unbiased bounded-variance stochastic first-order oracle, the $\epsilon^{-4}$ dependence matches the lower bound of~\citet{arjevani23lower}.
The leading terms coincide with those established for exact-polar Muon~\citep{kovalev25orthogonalization,shen26convergence}, as summarized in~\cref{app:comparison}.
The resulting complexity has the same dependence as that of~\citet{kim26convergence}, except that the middle term improves from $r^2\sigma^2/\epsilon^2$ to $r\sigma^2/\epsilon^2$, though the algorithms under analysis are slightly different.
As in the analysis for nonsmooth objectives, the proof in~\cref{app:smooth} chooses $q$ by balancing the $q$-dependent terms in the regret bound, rather than by bounding the finite-step approximation error relative to the exact-polar update.

\subsection{Interpretation of Muon with finite Newton--Schulz as FTRL}
\label{sec:ftrl}

By Fenchel duality, the online learner~\cref{eq:general-spectral-action} is follow-the-regularized-leader (FTRL) on the discounted linear losses, with a spectral regularizer given by the Fenchel conjugate of the smoothed potential.
Define the Fenchel conjugates $\phi_h^*(a)\coloneqq\sup_{x\geq0}\{ax-\phi_h(x)\}$ and $\widetilde\Phi_{h,C}^*(W)\coloneqq\sup_{M\in\R^{m\times n}}\{\inpr{W,M}-\widetilde\Phi_{h,C}(M)\}$.

\begin{theorem}
\label{thm:general-ftrl}
Suppose~\cref{ass:general-spectral-link} holds and let $(C_t)_{t\geq1}$ be a positive sequence.
For any $t\geq1$, the action~\cref{eq:general-spectral-action} is the FTRL update
\begin{equation}
  X_t^h
  =-D\nabla\widetilde\Phi_{h,C_t}(M_{t-1})
  \in\argmin_{\lVert X\rVert_{\mathrm{op}}\leq D}
  \left\{
    \inpr*{\sum_{s=1}^{t-1}\beta^{-s}G_s,X}
    +\frac{R_{h,C_t}(X)}{(1-\beta)\beta^{t-1}}
  \right\},
  \label{eq:general-potential-action}
\end{equation}
where the regularizer $R_{h,C}$ is given by
\begin{equation}
  R_{h,C}(X)
  \coloneqq D\widetilde\Phi_{h,C}^*(X/D)
  =DC\sum_{i=1}^r
  \phi_h^*\left(\frac{\sigma_i(X)}{D}\right).
  \label{eq:general-ftrl-regularizer}
\end{equation}
\end{theorem}

By~\cref{thm:general-ftrl} with $h=h_q$, the finite Newton--Schulz update of~\cref{alg:momentum-muon} corresponds to FTRL with the regularizer $R_{h_q,C_t}(X)$ controlled by the depth $q$.
At $q=0$, the regularizer is the Euclidean quadratic $(C_t/2D)\lVert X\rVert_{\mathrm F}^2$ on the ball, and as $q\to\infty$ it decays to zero on the ball, and the update approaches the follow-the-leader action $-D\polar(M_{t-1})$ in~\cref{eq:exact-polar-ftl}.
We discuss this in more detail in~\Cref{app:ftrl}.

\section{Conclusion and future work}
\label{sec:conclusion}

This paper analyzed Muon with momentum and finite Newton--Schulz through the discounted online-to-nonconvex conversion, viewing it as an online learner with a smoothed spectral potential.
In this view, the finite Newton--Schulz iteration is a smoothing mechanism rather than an approximation error, and the depth of the Newton--Schulz iteration governs a penalty--stability tradeoff between the polar-approximation error and the Lipschitz constant of the update.
Balancing the two terms at a depth that grows only logarithmically in the target accuracy yields sublinear discounted regret, and hence convergence to stationary points of nonsmooth nonconvex objectives, whereas Muon with the exact-polar update may fail to converge.
For smooth objectives, the same reduction yields rates whose leading terms coincide with those established for exact-polar Muon, and the regret bound extends to general spectral maps.

The analysis leaves a gap to deployed Muon.
It fixes the normalization to a constant known before the run, whereas practical Muon normalizes by a data-dependent quantity such as the current Frobenius norm of the momentum, which can violate the scale condition of~\cref{thm:general-discounted-regret}.
Our guarantees also assume the classical coefficients obtained from the truncation of the Taylor expansion, rather than the empirically tuned quintic used in the original Muon implementation of~\citet{jordan2024muon}.
Moreover, the conversion queries the gradient at random points between consecutive iterates, whereas practical Muon evaluates it at the current iterate.
Extending the analysis to these practical variants is left for future work.

\bibliographystyle{plainnat}
\bibliography{ref}

\begin{thebibliography}{59}
\providecommand{\natexlab}[1]{#1}
\providecommand{\url}[1]{\texttt{#1}}
\expandafter\ifx\csname urlstyle\endcsname\relax
  \providecommand{\doi}[1]{doi: #1}\else
  \providecommand{\doi}{doi: \begingroup \urlstyle{rm}\Url}\fi

\bibitem[Abernethy et~al.(2014)Abernethy, Lee, Sinha, and
  Tewari]{abernethy14smoothing}
Jacob Abernethy, Chansoo Lee, Abhinav Sinha, and Ambuj Tewari.
\newblock Online linear optimization via smoothing.
\newblock In \emph{Proceedings of the 27th Conference on Learning Theory},
  volume~35, pages 807--823, 2014.

\bibitem[Abernethy et~al.(2016)Abernethy, Lee, and
  Tewari]{abernethy16perturbation}
Jacob Abernethy, Chansoo Lee, and Ambuj Tewari.
\newblock Perturbation techniques in online learning and optimization.
\newblock In \emph{Perturbations, Optimization, and Statistics}, chapter~8,
  pages 233--264. MIT Press, 2016.

\bibitem[Ahn and Cutkosky(2024)]{ahn24model}
Kwangjun Ahn and Ashok Cutkosky.
\newblock {Adam} with model exponential moving average is effective for
  nonconvex optimization.
\newblock In \emph{Advances in Neural Information Processing Systems},
  volume~37, pages 94909--94933, 2024.

\bibitem[Ahn et~al.(2024)Ahn, Zhang, Kook, and Dai]{ahn24understanding}
Kwangjun Ahn, Zhiyu Zhang, Yunbum Kook, and Yan Dai.
\newblock Understanding {Adam} optimizer via online learning of updates: {Adam}
  is {FTRL} in disguise.
\newblock In \emph{Proceedings of the 41st International Conference on Machine
  Learning}, volume 235, pages 619--640, 2024.

\bibitem[Ahn et~al.(2025)Ahn, Magakyan, and Cutkosky]{ahn25general}
Kwangjun Ahn, Gagik Magakyan, and Ashok Cutkosky.
\newblock General framework for online-to-nonconvex conversion: Schedule-free
  {SGD} is also effective for nonconvex optimization.
\newblock In \emph{Proceedings of the 42nd International Conference on Machine
  Learning}, volume 267, pages 772--795, 2025.

\bibitem[Amsel et~al.(2026)Amsel, Persson, Musco, and Gower]{amsel26polar}
Noah Amsel, David Persson, Christopher Musco, and Robert~M. Gower.
\newblock The {Polar} {Express}: Optimal matrix sign methods and their
  application to the {Muon} algorithm.
\newblock In \emph{International Conference on Learning Representations}, 2026.

\bibitem[Andersson et~al.(2016)Andersson, Carlsson, and
  Perfekt]{andersson16operator}
Fredrik Andersson, Marcus Carlsson, and Karl-Mikael Perfekt.
\newblock {Operator-Lipschitz} estimates for the singular value functional
  calculus.
\newblock \emph{Proceedings of the American Mathematical Society}, 144\penalty0
  (5):\penalty0 1867--1875, 2016.

\bibitem[Arjevani et~al.(2023)Arjevani, Carmon, Duchi, Foster, Srebro, and
  Woodworth]{arjevani23lower}
Yossi Arjevani, Yair Carmon, John~C. Duchi, Dylan~J. Foster, Nathan Srebro, and
  Blake Woodworth.
\newblock Lower bounds for non-convex stochastic optimization.
\newblock \emph{Mathematical Programming}, 199\penalty0 (1--2):\penalty0
  165--214, 2023.

\bibitem[Bernstein and Newhouse(2024)]{bernstein24old}
Jeremy Bernstein and Laker Newhouse.
\newblock Old optimizer, new norm: An anthology.
\newblock In \emph{OPT 2024: 16th Annual Workshop on Optimization for Machine
  Learning}, 2024.

\bibitem[Bernstein and Newhouse(2025)]{bernstein25modular}
Jeremy Bernstein and Laker Newhouse.
\newblock Modular duality in deep learning.
\newblock In \emph{Proceedings of the 42nd International Conference on Machine
  Learning}, volume 267, pages 3920--3930, 2025.

\bibitem[Bj{\"o}rck and Bowie(1971)]{bjorck71iterative}
{\AA}ke Bj{\"o}rck and C.~Bowie.
\newblock An iterative algorithm for computing the best estimate of an
  orthogonal matrix.
\newblock \emph{SIAM Journal on Numerical Analysis}, 8\penalty0 (2):\penalty0
  358--364, 1971.

\bibitem[Braun et~al.(2026)Braun, Bao, Huang, and Imaizumi]{braun26spectral}
Guillaume Braun, Han Bao, Wei Huang, and Masaaki Imaizumi.
\newblock Spectral gradient descent mitigates anisotropy-driven misalignment: A
  case study in phase retrieval.
\newblock In \emph{International Conference on Machine Learning}, 2026.

\bibitem[Carlson et~al.(2015)Carlson, Collins, Hsieh, Carin, and
  Cevher]{carlson15spectral}
David~E. Carlson, Edo Collins, Ya-Ping Hsieh, Lawrence Carin, and Volkan
  Cevher.
\newblock Preconditioned spectral descent for deep learning.
\newblock In \emph{Advances in Neural Information Processing Systems},
  volume~28, pages 2971--2979, 2015.

\bibitem[Carmon et~al.(2020)Carmon, Duchi, Hinder, and Sidford]{carmon20lower}
Yair Carmon, John~C. Duchi, Oliver Hinder, and Aaron Sidford.
\newblock Lower bounds for finding stationary points {I}.
\newblock \emph{Mathematical Programming}, 184:\penalty0 71--120, 2020.

\bibitem[Chen et~al.(2026)Chen, Li, and Liu]{chen26spectral}
Lizhang Chen, Jonathan Li, and Qiang Liu.
\newblock {Muon} optimizes under spectral norm constraints.
\newblock \emph{Transactions on Machine Learning Research}, 2026.

\bibitem[Choudhury et~al.(2026)Choudhury, Cheng, Tak{\'a}{\v c}, Na, and
  Kolar]{choudhury26nesterov}
Sayantan Choudhury, Xiaoran Cheng, Martin Tak{\'a}{\v c}, Sen Na, and Mladen
  Kolar.
\newblock {Muon} with {Nesterov} momentum: Heavy-tailed noise and (randomized)
  inexact polar decomposition.
\newblock \emph{arXiv preprint arXiv:2605.06884}, 2026.

\bibitem[Cutkosky et~al.(2023)Cutkosky, Mehta, and Orabona]{cutkosky23optimal}
Ashok Cutkosky, Harsh Mehta, and Francesco Orabona.
\newblock Optimal stochastic non-smooth non-convex optimization through
  online-to-non-convex conversion.
\newblock In \emph{Proceedings of the 40th International Conference on Machine
  Learning}, volume 202, pages 6643--6670, 2023.

\bibitem[Davis and Drusvyatskiy(2025)]{davis25spectral}
Damek Davis and Dmitriy Drusvyatskiy.
\newblock When do spectral gradient updates help in deep learning?
\newblock \emph{arXiv preprint arXiv:2512.04299}, 2025.

\bibitem[Dong and Sawin(2026)]{dong26fractional}
Yihe Dong and Will Sawin.
\newblock {Muon}$^p$: {Muon} with fractional spectral powers.
\newblock \emph{arXiv preprint arXiv:2606.13867}, 2026.

\bibitem[Fan et~al.(2025)Fan, Schmidt, and Thrampoulidis]{fan25implicit}
Chen Fan, Mark Schmidt, and Christos Thrampoulidis.
\newblock Implicit bias of spectral descent and {Muon} on multiclass separable
  data.
\newblock In \emph{Advances in Neural Information Processing Systems},
  volume~38, 2025.

\bibitem[Feoktistov et~al.(2026)Feoktistov, Belinsky, Veprikov, Zainullin, and
  Beznosikov]{feoktistov26softsign}
Dmitrii Feoktistov, Timofey Belinsky, Andrey Veprikov, Amir Zainullin, and
  Aleksandr Beznosikov.
\newblock Softsign: Smooth sign in your optimizer for better parameter
  heterogeneity handling.
\newblock In \emph{International Conference on Machine Learning}, 2026.

\bibitem[{GLM-4.5 Team}(2025)]{glm25glm}
{GLM-4.5 Team}.
\newblock {GLM-4.5}: Agentic, reasoning, and coding ({ARC}) foundation models.
\newblock \emph{arXiv preprint arXiv:2508.06471}, 2025.

\bibitem[Goldstein(1977)]{goldstein77optimization}
A.~A. Goldstein.
\newblock Optimization of {Lipschitz} continuous functions.
\newblock \emph{Mathematical Programming}, 13\penalty0 (1):\penalty0 14--22,
  1977.

\bibitem[Gonon et~al.(2026)Gonon, Mu{\c s}at, and Boumal]{gonon26insights}
Antoine Gonon, Andreea-Alexandra Mu{\c s}at, and Nicolas Boumal.
\newblock Insights on {Muon} from simple quadratics.
\newblock \emph{arXiv preprint arXiv:2602.11948}, 2026.

\bibitem[Ji and Yuan(2026)]{ji26derandomized}
Fanfan Ji and Xiaotong Yuan.
\newblock Derandomized online-to-non-convex conversion for stochastic weakly
  convex optimization.
\newblock In \emph{International Conference on Learning Representations}, 2026.

\bibitem[Jiang et~al.(2026{\natexlab{a}})Jiang, Mhammedi, Mohri, and
  Mokhtari]{jiang26adaptive}
Ruichen Jiang, Zakaria Mhammedi, Mehryar Mohri, and Aryan Mokhtari.
\newblock Adaptive matrix online learning through smoothing with guarantees for
  nonsmooth nonconvex optimization.
\newblock In \emph{Proceedings of Thirty Ninth Conference on Learning Theory},
  volume 336, pages 3782--3824, 2026{\natexlab{a}}.

\bibitem[Jiang et~al.(2026{\natexlab{b}})Jiang, Semenov, and
  Stich]{jiang26clipping}
Xiaowen Jiang, Andrei Semenov, and Sebastian~U. Stich.
\newblock Enhancing {LLM} training via spectral clipping.
\newblock In \emph{International Conference on Machine Learning},
  2026{\natexlab{b}}.

\bibitem[Jordan et~al.(2024)Jordan, Jin, Boza, You, Cesista, Newhouse, and
  Bernstein]{jordan2024muon}
Keller Jordan, Yuchen Jin, Vlado Boza, Jiacheng You, Franz Cesista, Laker
  Newhouse, and Jeremy Bernstein.
\newblock Muon: An optimizer for hidden layers in neural networks, 2024.
\newblock URL \url{https://kellerjordan.github.io/posts/muon/}.

\bibitem[Kim and Oh(2026)]{kim26convergence}
Gyu~Yeol Kim and Min-hwan Oh.
\newblock Convergence of {Muon} with {Newton--Schulz}.
\newblock In \emph{International Conference on Learning Representations}, 2026.

\bibitem[{Kimi Team}(2025)]{kimi25k2}
{Kimi Team}.
\newblock {Kimi} {K2}: Open agentic intelligence.
\newblock \emph{arXiv preprint arXiv:2507.20534}, 2025.

\bibitem[Kornowski and Shamir(2022)]{kornowski22oracle}
Guy Kornowski and Ohad Shamir.
\newblock Oracle complexity in nonsmooth nonconvex optimization.
\newblock \emph{Journal of Machine Learning Research}, 23\penalty0
  (314):\penalty0 1--44, 2022.

\bibitem[Kovalev(2025)]{kovalev25orthogonalization}
Dmitry Kovalev.
\newblock Understanding gradient orthogonalization for deep learning via
  non-{Euclidean} trust-region optimization.
\newblock \emph{arXiv preprint arXiv:2503.12645}, 2025.

\bibitem[Kovarik(1970)]{kovarik70iterative}
Zdislav Kovarik.
\newblock Some iterative methods for improving orthonormality.
\newblock \emph{SIAM Journal on Numerical Analysis}, 7\penalty0 (3):\penalty0
  386--389, 1970.

\bibitem[Lewis(1995)]{lewis95convex}
Adrian~S. Lewis.
\newblock The convex analysis of unitarily invariant matrix functions.
\newblock \emph{Journal of Convex Analysis}, 2\penalty0 (1--2):\penalty0
  173--183, 1995.

\bibitem[Li and Hong(2025)]{li25note}
Jiaxiang Li and Mingyi Hong.
\newblock A note on the convergence of {Muon}.
\newblock \emph{arXiv preprint arXiv:2502.02900}, 2025.

\bibitem[Li et~al.(2026)Li, Zhang, Liu, and Bao]{li26denoise}
Xianliang Li, Zihan Zhang, Weiyang Liu, and Han Bao.
\newblock Denoise first, orthogonalize later: Understanding momentum in {Muon}
  via spectral filtering.
\newblock \emph{arXiv preprint arXiv:2606.03899}, 2026.

\bibitem[Liu et~al.(2025)Liu, Su, Yao, Jiang, Lai, Du, Qin, Xu, Lu, Yan,
  et~al.]{liu25scalable}
Jingyuan Liu, Jianlin Su, Xingcheng Yao, Zhejun Jiang, Guokun Lai, Yulun Du,
  Yidao Qin, Weixin Xu, Enzhe Lu, Junjie Yan, et~al.
\newblock {Muon} is scalable for {LLM} training.
\newblock \emph{arXiv preprint arXiv:2502.16982}, 2025.

\bibitem[Liu et~al.(2024)Liu, Wang, and Zhang]{liu24highprobability}
Langqi Liu, Yibo Wang, and Lijun Zhang.
\newblock High-probability bound for non-smooth non-convex stochastic
  optimization with heavy tails.
\newblock In \emph{Proceedings of the 41st International Conference on Machine
  Learning}, volume 235, pages 32122--32138, 2024.

\bibitem[Liu(2026)]{liu26online}
Zijian Liu.
\newblock Online convex optimization with heavy tails: Old algorithms, new
  regrets, and applications.
\newblock In \emph{Proceedings of The 37th International Conference on
  Algorithmic Learning Theory}, volume 313, pages 1--47, 2026.

\bibitem[Loshchilov and Hutter(2019)]{loshchilov19decoupled}
Ilya Loshchilov and Frank Hutter.
\newblock Decoupled weight decay regularization.
\newblock In \emph{International Conference on Learning Representations}, 2019.

\bibitem[Ma et~al.(2026)Ma, Huang, Chi, and Chen]{ma26preconditioning}
Jianhao Ma, Yu~Huang, Yuejie Chi, and Yuxin Chen.
\newblock Preconditioning benefits of spectral orthogonalization in {Muon}.
\newblock \emph{arXiv preprint arXiv:2601.13474}, 2026.

\bibitem[McMahan(2017)]{mcmahan17survey}
H.~Brendan McMahan.
\newblock A survey of algorithms and analysis for adaptive online learning.
\newblock \emph{Journal of Machine Learning Research}, 18\penalty0
  (90):\penalty0 1--50, 2017.

\bibitem[Mustafi et~al.(2026)Mustafi, Mukherjee, and
  Sriperumbudur]{mustafi2026move}
Aratrika Mustafi, Soumya Mukherjee, and Bharath~K. Sriperumbudur.
\newblock Move on {Muon}: A {Hamiltonian} probability gradient flow perspective
  of {Muon} optimizer.
\newblock \emph{arXiv preprint arXiv:2605.23871}, 2026.

\bibitem[Parshakova et~al.(2026)Parshakova, Khaled, Crawshaw, Garrigos, and
  Gower]{parshakova26muon}
Tetiana Parshakova, Ahmed Khaled, Michael Crawshaw, Guillaume Garrigos, and
  Robert~M. Gower.
\newblock {Muon} does not converge on convex {Lipschitz} functions.
\newblock \emph{arXiv preprint arXiv:2605.08980}, 2026.

\bibitem[Patitucci et~al.(2026)Patitucci, Jiang, and
  Mokhtari]{patitucci26improving}
Francisco Patitucci, Ruichen Jiang, and Aryan Mokhtari.
\newblock Improving online-to-nonconvex conversion for smooth optimization via
  double optimism.
\newblock In \emph{International Conference on Learning Representations}, 2026.

\bibitem[Pethick et~al.(2025)Pethick, Xie, Antonakopoulos, Zhu, Silveti-Falls,
  and Cevher]{pethick25lmo}
Thomas Pethick, Wanyun Xie, Kimon Antonakopoulos, Zhenyu Zhu, Antonio
  Silveti-Falls, and Volkan Cevher.
\newblock Training deep learning models with norm-constrained {LMOs}.
\newblock In \emph{Proceedings of the 42nd International Conference on Machine
  Learning}, volume 267, pages 49069--49104, 2025.

\bibitem[Qi et~al.(2026)Qi, Chen, Ye, He, and Xiao]{qi26delving}
Xianbiao Qi, Marco Chen, Jiaquan Ye, Yelin He, and Rong Xiao.
\newblock Delving into {Muon} and beyond: Deep analysis and extensions.
\newblock In \emph{International Conference on Machine Learning}, 2026.

\bibitem[Riabinin et~al.(2026)Riabinin, Shulgin, Gruntkowska, and
  Richt{\'a}rik]{riabinin26gluon}
Artem Riabinin, Egor Shulgin, Kaja Gruntkowska, and Peter Richt{\'a}rik.
\newblock From {Muon} to {Gluon}: Bridging theory and practice of {LMO}-based
  optimizers for {LLMs}.
\newblock In \emph{International Conference on Machine Learning}, 2026.

\bibitem[Sfyraki and Wang(2026)]{sfyraki26lions}
Maria-Eleni Sfyraki and Jun-Kun Wang.
\newblock Lions and {Muons}: Optimization via stochastic {Frank--Wolfe} under
  heavy-tailed noise.
\newblock In \emph{International Conference on Machine Learning}, 2026.

\bibitem[Shalev-Shwartz(2012)]{shalevshwartz12online}
Shai Shalev-Shwartz.
\newblock Online learning and online convex optimization.
\newblock \emph{Foundations and Trends in Machine Learning}, 4\penalty0
  (2):\penalty0 107--194, 2012.

\bibitem[Shen et~al.(2026)Shen, Huang, Huang, Shen, and
  Zhang]{shen26convergence}
Wei Shen, Ruichuan Huang, Minhui Huang, Cong Shen, and Jiawei Zhang.
\newblock On the convergence analysis of {Muon}.
\newblock \emph{Transactions on Machine Learning Research}, 2026.

\bibitem[Shulgin et~al.(2026{\natexlab{a}})Shulgin, Alrashed, Orabona, and
  Richt{\'a}rik]{shulgin26inexact}
Egor Shulgin, Sultan Alrashed, Francesco Orabona, and Peter Richt{\'a}rik.
\newblock Beyond the ideal: Analyzing the inexact {Muon} update.
\newblock In \emph{International Conference on Artificial Intelligence and
  Statistics}, volume 300, 2026{\natexlab{a}}.

\bibitem[Shulgin et~al.(2026{\natexlab{b}})Shulgin, Laing, Orvieto, and
  Richt{\'a}rik]{shulgin26quadratic}
Egor Shulgin, Sam Laing, Antonio Orvieto, and Peter Richt{\'a}rik.
\newblock A quadratic lens on {Muon}: Orthogonalization, invariance, and
  implicit preconditioning.
\newblock In \emph{HiLD 2026: 4th Workshop on High-dimensional Learning
  Dynamics}, 2026{\natexlab{b}}.

\bibitem[Tsuchiya et~al.(2023)Tsuchiya, Ito, and Honda]{tsuchiya23stability}
Taira Tsuchiya, Shinji Ito, and Junya Honda.
\newblock Stability-penalty-adaptive follow-the-regularized-leader: Sparsity,
  game-dependency, and best-of-both-worlds.
\newblock In \emph{Advances in Neural Information Processing Systems},
  volume~36, pages 47406--47437, 2023.

\bibitem[Wen et~al.(2026)Wen, Hall, Ma, and Liang]{wen26fantastic}
Kaiyue Wen, David Hall, Tengyu Ma, and Percy Liang.
\newblock Fantastic pretraining optimizers and where to find them.
\newblock In \emph{International Conference on Learning Representations}, 2026.

\bibitem[Wu et~al.(2026)Wu, Shah, Silwal, and Zhang]{wu26dynmuon}
Fangzhou Wu, Rikhav Shah, Sandeep Silwal, and Qiuyi Zhang.
\newblock {DynMuon}: A dynamic spectral shaping view of {Muon}.
\newblock In \emph{HiLD 2026: 4th Workshop on High-dimensional Learning
  Dynamics}, 2026.

\bibitem[Zhang et~al.(2020)Zhang, Lin, Jegelka, Sra, and
  Jadbabaie]{zhang20complexity}
Jingzhao Zhang, Hongzhou Lin, Stefanie Jegelka, Suvrit Sra, and Ali Jadbabaie.
\newblock Complexity of finding stationary points of nonconvex nonsmooth
  functions.
\newblock In \emph{Proceedings of the 37th International Conference on Machine
  Learning}, volume 119, pages 11173--11182, 2020.

\bibitem[Zhang and Cutkosky(2024)]{zhang24random}
Qinzi Zhang and Ashok Cutkosky.
\newblock Random scaling and momentum for non-smooth non-convex optimization.
\newblock In \emph{Proceedings of the 41st International Conference on Machine
  Learning}, volume 235, pages 58780--58799, 2024.

\bibitem[Zimmert and Seldin(2021)]{zimmert21tsallis}
Julian Zimmert and Yevgeny Seldin.
\newblock {Tsallis-INF}: An optimal algorithm for stochastic and adversarial
  bandits.
\newblock \emph{Journal of Machine Learning Research}, 22\penalty0
  (28):\penalty0 1--49, 2021.

\end{thebibliography}

\crefalias{section}{appendix}
\crefalias{subsection}{appendix}
\crefalias{subsubsection}{appendix}

\newpage

\appendix

\section{Comparison with prior convergence guarantees}
\label{app:comparison}

This appendix compares the guarantees of this work with those of prior work on Muon and related methods.

\begin{table}[htbp]
\centering\footnotesize
\renewcommand{\arraystretch}{1.25}\setlength{\tabcolsep}{3.5pt}
\caption{Gradient-oracle complexity for nonsmooth nonconvex objectives under the $(\rho,\epsilon)$-stationarity criterion of~\cref{def:rho-stationarity}.
Here $r=m$, $\nu=(n/(n-r-1))^{1/4}$, and NS abbreviates Newton--Schulz.
The bounds for Pion and Leon are stated for sufficiently small $\epsilon$.}
\label{tab:nonsmooth-comparison}
\begin{tabular}{llll}
\toprule
Reference & Smoothing & Polar computation & Gradient-oracle complexity \\
\midrule
Exact-polar Muon
& None
& Exact polar
& Not covered by discounted O2NC \\
\addlinespace
  Pion~\citep{jiang26adaptive}
& Gaussian perturbation
& Exact polar
& $\displaystyle O\left(\max\left\{\frac{\nu^2\lVert Q\rVert_*^2\Delta_{\mathcal L}}{\rho\epsilon^3},\frac{\nu^2\lVert Q\rVert_*^2}{\epsilon^2},\frac{\nu rG}{\epsilon}\right\}\right)$ \\
\addlinespace
Leon~\citep{jiang26adaptive}
& Hyperbolic smoothing
& Exact polar
& $\displaystyle O\left(\max\left\{\frac{\lVert Q\rVert_*^2\Delta_{\mathcal L}}{\rho\epsilon^3},\frac{\lVert Q\rVert_*^2}{\epsilon^2},\frac{rG}{\epsilon}\right\}\right)$ \\
\midrule
  \textbf{This work} (\cref{thm:ns-complexity})
& Finite NS
& $q=O(\log(1/\epsilon))$ NS steps
& $\displaystyle O\left(\max\left\{\frac{r(\Gamma+\sigma)^2\Delta_{\mathcal L}}{\rho\epsilon^3},\frac{r(\Gamma+\sigma)^2}{\epsilon^2}\right\}\right)$ \\
\bottomrule
\end{tabular}
\end{table}

\begin{table}[htbp]
\centering\footnotesize
\renewcommand{\arraystretch}{1.25}\setlength{\tabcolsep}{3.5pt}
\caption{Gradient-oracle complexity for finding $\epsilon$-stationary points of smooth nonconvex objectives.
The results are specialized to smoothness in the operator norm and stationarity in the nuclear norm.
The ``Deterministic'' and ``Stochastic'' columns give the complexity with exact gradients and with unbiased stochastic gradients of variance at most $\sigma^2$, respectively, and the bounds for the stochastic setting assume one sample per oracle query.
For~\citet{shulgin26inexact}, $\delta<1$ is a uniform additive error of the inexact linear minimization oracle.
For~\citet{kim26convergence}, $\chi_q=(1-\epsilon_q)^{-1}$, where $\epsilon_q$ is the uniform operator-norm error of the $q$-step Newton--Schulz output relative to the polar factor, and under their uniform condition on the initial residual, $\chi_q$ converges to one doubly exponentially in $q$.
NS abbreviates Newton--Schulz.}
\label{tab:smooth-comparison}
\begin{tabular}{llll}
\toprule
Reference & Deterministic & Stochastic & Polar computation \\
\midrule
\citet{kovalev25orthogonalization}
& $O\left(\displaystyle \frac{L_{\mathrm{op}}\Delta_{\mathcal L}}{\epsilon^2}\right)$
& $\displaystyle O\left(\max\left\{\frac{L_{\mathrm{op}}\Delta_{\mathcal L}}{\epsilon^2},\frac{\sqrt r\sigma}{\epsilon},\frac{r^{3/2}\sigma^3}{\epsilon^3},\frac{r\sigma^2L_{\mathrm{op}}\Delta_{\mathcal L}}{\epsilon^4}\right\}\right)$
& Exact polar \\
\addlinespace
\citet{shen26convergence}
& $\displaystyle O\left(\frac{L_{\mathrm{op}}\Delta_{\mathcal L}}{\epsilon^2}\right)$
& $\displaystyle O\left(\max\left\{\frac{L_{\mathrm{op}}\Delta_{\mathcal L}}{\epsilon^2},\frac{r^2\sigma^4}{L_{\mathrm{op}}\Delta_{\mathcal L}\epsilon^2},\frac{r\sigma^2L_{\mathrm{op}}\Delta_{\mathcal L}}{\epsilon^4}\right\}\right)$
& Exact polar \\
\addlinespace
\citet{shulgin26inexact}
& $\displaystyle O\left(\frac{(1+\delta)^2L_{\mathrm{op}}\Delta_{\mathcal L}}{(1-\delta)^2\epsilon^2}\right)$
& $\displaystyle O\left(\frac{(1+\delta)r\sigma^2L_{\mathrm{op}}\Delta_{\mathcal L}}{(1-\delta)^4\epsilon^4}\right)$
& Inexact LMO \\
\addlinespace
\citet{kim26convergence}
& $\displaystyle O\left(\frac{\chi_q^2L_{\mathrm{op}}\Delta_{\mathcal L}}{\epsilon^2}\right)$
& $\displaystyle O\left(\max\left\{\frac{\chi_q^2L_{\mathrm{op}}\Delta_{\mathcal L}}{\epsilon^2},\frac{\chi_q^2r^2\sigma^2}{\epsilon^2},\frac{\chi_q^4r\sigma^2L_{\mathrm{op}}\Delta_{\mathcal L}}{\epsilon^4}\right\}\right)$
& Finite NS \\
\midrule
  \textbf{This work}~(\cref{thm:smooth-stationarity})
& $\displaystyle O\left(\frac{L_{\mathrm{op}}\Delta_{\mathcal L}}{\epsilon^2}\right)$
& $\displaystyle O\left(\max\left\{\frac{L_{\mathrm{op}}\Delta_{\mathcal L}}{\epsilon^2},\frac{r\sigma^2}{\epsilon^2},\frac{r\sigma^2L_{\mathrm{op}}\Delta_{\mathcal L}}{\epsilon^4}\right\}\right)$
& Finite NS \\
\bottomrule
\end{tabular}
\end{table}

In~\cref{tab:nonsmooth-comparison}, we summarize the gradient-oracle complexity and the treatment of the polar computation for nonsmooth nonconvex objectives.
The leading $\rho^{-1}\epsilon^{-3}$ term in the bound of this work matches the bounds for Pion and Leon.
The bound for Pion assumes an exact expectation of polar factors, and the Monte Carlo implementation of Pion uses $k=T$ samples per round to preserve this bound.
For exact-polar Muon, the discounted O2NC framework alone does not provide a general stationarity guarantee for nonsmooth objectives, as discussed in~\cref{sec:o2nc}.
For Pion and Leon, the matrix $Q\succeq0$ and the constant $G>0$ satisfy $\E[G_tG_t^\top]\preceq Q^2$ and $\lVert G_t\rVert_{\mathrm{op}}\leq G$ almost surely.
The quantities $\lVert Q\rVert_*^2$ and $r(\Gamma+\sigma)^2$ reflect different oracle assumptions and are not directly comparable.

In~\cref{tab:smooth-comparison}, we summarize the gradient-oracle complexity and the treatment of the polar computation for smooth nonconvex objectives.
The leading terms of this work coincide with those established for exact-polar Muon~\citep{kovalev25orthogonalization,shen26convergence}.
Compared with~\citet{kim26convergence}, the bounds of this work involve no factor $\chi_q$, and the middle term improves from $r^2\sigma^2/\epsilon^2$ to $r\sigma^2/\epsilon^2$, though the algorithms under analysis are slightly different.
The output criteria follow the respective references: an expected minimum for~\citet{kovalev25orthogonalization}, an average expected gradient norm for the other prior works, and the randomized output in~\cref{thm:smooth-stationarity} for this work.
The bound for~\citet{shulgin26inexact} in the stochastic setting specializes their norm-compatibility factor to the nuclear norm, whose square is $r$.
Their guarantee for the inexact LMO requires a uniform error $\delta<1$, which is not automatically available for a fixed-depth polar approximation near rank deficiency.
This work additionally assumes an almost-sure operator-norm bound and a Frobenius second-moment bound for the stochastic gradients.

\section{Auxiliary lemmas on singular-value functions}
\label{app:auxiliary-singular-values}

For ease of reference, this appendix collects the standard singular-value lemmas used in this paper.
Some of them are stated in the special forms in which we use them.

The following three lemmas specialize results of \citet{lewis95convex} to separable functions of the singular values.

\begin{lemma}[Convexity of separable singular-value functions]
\label{lem:aux-sv-convexity}
Let $\phi\colon[0,\infty)\to\R$ be convex and nondecreasing.
Then $S\mapsto\sum_{i=1}^r\phi(\sigma_i(S))$ is convex on $\R^{m\times n}$.
\end{lemma}

\begin{proof}
Define $f\colon\R^r\to\R$ by $f(\gamma)\coloneqq\sum_{i=1}^r\phi(|\gamma_i|)$.
Since $\phi$ is convex and nondecreasing, the map $t\mapsto\phi(|t|)$ is convex.
Thus, $f$ is convex, lower semicontinuous, and absolutely symmetric.
Applying \citet[Corollary~2.6]{lewis95convex} to $f$ shows that the function
$S\mapsto(f\circ\sigma)(S)=\sum_{i=1}^r\phi(\sigma_i(S))$ is convex, which completes the proof.
\end{proof}

The next lemma gives the corresponding gradient formula.
\begin{lemma}[{Singular-value gradient formula}]
\label{lem:aux-sv-gradient}
Let $\phi\colon[0,\infty)\to\R$ be convex and differentiable with $\phi'(0)=0$, and define $F(S)\coloneqq\sum_{i=1}^r\phi(\sigma_i(S))$.
Then $F$ is differentiable on all of $\R^{m\times n}$, and for any thin singular value decomposition $S=U\diag(\sigma_1(S),\ldots,\sigma_r(S))V^\top$,
\[
  \nabla F(S)
  =U\diag\left(\phi'(\sigma_1(S)),\ldots,\phi'(\sigma_r(S))\right)V^\top.
\]
\end{lemma}

\begin{proof}
Define $f\colon\R^r\to\R$ by $f(\gamma)\coloneqq\sum_{i=1}^r\phi(|\gamma_i|)$.
Since $\phi$ is convex and $\phi'(0)=0$, it is nondecreasing, and hence $f$ is convex and absolutely symmetric.
The condition $\phi'(0)=0$ also makes $t\mapsto\phi(|t|)$ differentiable at $0$, so $f$ is differentiable on $\R^r$ and
$\nabla f(\sigma(S))=(\phi'(\sigma_1(S)),\ldots,\phi'(\sigma_r(S)))$.
Applying \citet[Theorem~3.1]{lewis95convex} to $f$ shows that $F=f\circ\sigma$ is differentiable and has the above gradient, which completes the proof.
\end{proof}

The following lemma gives the Fenchel conjugate used in the FTRL interpretation.
\begin{lemma}[{Fenchel conjugate formula for singular-value functions}]
\label{lem:aux-conjugacy}
Let $\phi\colon[0,\infty)\to(-\infty,+\infty]$ be proper and convex, and write $\phi^*(a)\coloneqq\sup_{x\geq0}\{ax-\phi(x)\}$.
Then the Fenchel conjugate of $S\mapsto\sum_{i=1}^r\phi(\sigma_i(S))$ is $W\mapsto\sum_{i=1}^r\phi^*(\sigma_i(W))$.
\end{lemma}

\begin{proof}
Define $f\colon\R^r\to(-\infty,+\infty]$ and
$F\colon\R^{m\times n}\to(-\infty,+\infty]$ by
$f(\gamma)\coloneqq\sum_{i=1}^r\phi(|\gamma_i|)$ and
$F(S)\coloneqq\sum_{i=1}^r\phi(\sigma_i(S))$.
The function $f$ is invariant under sign changes and coordinate permutations, and $F=f\circ\sigma$.
Hence, \citet[Theorem~2.4]{lewis95convex} gives $F^*=f^*\circ\sigma$.
For every $\mu\in\R^r$,
\begin{align*}
  f^*(\mu)
  &=
  \sup_{\gamma\in\R^r}
  \sum_{i=1}^r\{\mu_i\gamma_i-\phi(|\gamma_i|)\}
  =
  \sum_{i=1}^r
  \sup_{t\in\R}\{\mu_i t-\phi(|t|)\}\\
  &=
  \sum_{i=1}^r
  \sup_{x\geq0}\{|\mu_i|x-\phi(x)\}
  =
  \sum_{i=1}^r\phi^*(|\mu_i|).
\end{align*}
Since $\sigma_i(W)\geq0$, we obtain $F^*(W)=f^*(\sigma(W))=\sum_{i=1}^r\phi^*(\sigma_i(W))$, which completes the proof.
\end{proof}

The following lemma for singular-value maps is used to control the Lipschitz constant of the gradient of the spectral potential.
\begin{lemma}[{Lipschitz bound for singular-value maps, \citealp[Proposition~4.1 and Theorem~4.2]{andersson16operator}}]
\label{lem:aux-sv-lipschitz}
Let $h\colon[0,\infty)\to\R$ satisfy $h(0)=0$ and be $\Lip(h)$-Lipschitz.
Then the singular-value map $\calH_h$ of~\cref{eq:singular-value-map} is well defined and, for all $X,Y\in\R^{m\times n}$,
\[
  \lVert\calH_h(X)-\calH_h(Y)\rVert_{\mathrm F}
  \leq \Lip(h) \lVert X-Y\rVert_{\mathrm F}.
\]
\end{lemma}
\section{Proof of~\Cref{lem:jiang-o2nc}}
\label{app:o2nc-proof}

This appendix provides the proof of~\Cref{lem:jiang-o2nc}.

\begin{proof}
Let $\mathcal F_0$ be the $\sigma$-algebra generated by the learner's initialization and its internal randomness, and, for $t\geq1$, let $\mathcal F_t$ be the $\sigma$-algebra generated by $\mathcal F_0$ and all randomness through round $t$.
By construction, $W_t$ and $X_t$ are $\mathcal F_{t-1}$-measurable.
Let $E_t\coloneqq G_t-\nabla\mathcal L(\widetilde W_t)$ denote the oracle noise.
Specializing the proof of Proposition~25 of~\citet{jiang26adaptive} to the operator and nuclear norms and using $\mathcal L(W_T)\geq\inf_W\mathcal L(W)$ gives
\begin{align*}
  \E\brk*{\lVert\nabla\mathcal L(\bar W_\tau)\rVert_*^{[\rho]}}
  &\leq
  \frac{4\Delta_{\mathcal L}}{(1-\beta)\rho T}
  +\frac1T\E\left[
      \DReg_T(1)+(1-\beta)\sum_{t=1}^{T-1}\DReg_t(1)
    \right]\\
  &\qquad
  +\frac1T\E\brk*{\left\lVert\sum_{t=1}^T\beta^{T-t}E_t\right\rVert_*}
  +\frac{1-\beta}{T}\sum_{t=1}^{T-1}
    \E\brk*{\left\lVert\sum_{s=1}^t\beta^{t-s}E_s\right\rVert_*}.
\end{align*}
Each $E_t$ is $\mathcal F_t$-measurable.
Conditional on $\mathcal F_{t-1}$ and $u_t$, \cref{eq:oracle-bounds} of~\cref{ass:oracle} gives $\E[E_t\mid\mathcal F_{t-1},u_t]=0$ and $\E[\lVert E_t\rVert_{\mathrm{F}}^2\mid\mathcal F_{t-1},u_t]\leq\sigma^2$.
The tower property therefore gives $\E[E_t\mid\mathcal F_{t-1}]=0$ and $\E[\lVert E_t\rVert_{\mathrm{F}}^2]\leq\sigma^2$, so $(E_t)_{t\geq1}$ is a matrix-valued martingale difference sequence.
Consequently, the conditional orthogonality of the martingale differences gives, for every $t\geq1$,
\[
  \E\brk*{\left\lVert\sum_{s=1}^t\beta^{t-s}E_s\right\rVert_{\mathrm{F}}^2}
  =\sum_{s=1}^t\beta^{2(t-s)}\E\brk*{\lVert E_s\rVert_{\mathrm{F}}^2}
  \leq\sigma^2\frac{1-\beta^{2t}}{1-\beta^2}
  \leq\frac{\sigma^2}{1-\beta^2}.
\]
The inequality $\lVert\cdot\rVert_*\leq\sqrt r\,\lVert\cdot\rVert_{\mathrm{F}}$ and Jensen's inequality give, for every $t\geq1$,
\[
  \E\brk*{\left\lVert\sum_{s=1}^t\beta^{t-s}E_s\right\rVert_*}
  \leq\sqrt r\left(\E\brk*{\left\lVert\sum_{s=1}^t\beta^{t-s}E_s\right\rVert_{\mathrm{F}}^2}\right)^{1/2}
  \leq\frac{\sqrt r\,\sigma}{\sqrt{1-\beta^2}}.
\]
Therefore,
\begin{align*}
  &\frac1T\E\brk*{\left\lVert\sum_{t=1}^T\beta^{T-t}E_t\right\rVert_*}
  +\frac{1-\beta}{T}\sum_{t=1}^{T-1}
    \E\brk*{\left\lVert\sum_{s=1}^t\beta^{t-s}E_s\right\rVert_*}
  \\
  &\leq
  \left(\frac1T+\frac{(1-\beta)(T-1)}{T}\right)
  \frac{\sqrt r\,\sigma}{\sqrt{1-\beta^2}}
  =
  \left(1-\beta+\frac{\beta}{T}\right)
  \frac{\sqrt r\,\sigma}{\sqrt{1-\beta^2}}.
\end{align*}
Combining the preceding bounds completes the proof.
\end{proof}

\section{Deferred proofs from~\Cref{sec:ns}}
\label{app:ns-proofs}

This appendix provides the proofs deferred from~\cref{sec:ns}.

\subsection{Newton--Schulz iterations at a general Taylor order}
The analysis in~\cref{sec:ns} specializes to the degree-five Newton--Schulz polynomial used in~\cref{alg:momentum-muon}.
We begin by defining the Newton--Schulz iteration obtained by truncating the Taylor expansion at an arbitrary order $\kappa\geq1$ and then show that the same penalty--stability tradeoff holds for every such order. 
For an integer $\kappa\geq1$, recall from~\cref{eq:ns-taylor} in~\cref{sec:setup} the polynomial $p_\kappa$ and coefficients $c_s$, and define
\[
f_\kappa(x)
\coloneqq
x p_\kappa(x^2)
=
x\sum_{s=0}^{\kappa}
c_s(1-x^2)^s,
\qquad x\in[0,1].
\]
The corresponding Newton--Schulz iteration is $x_{j+1}=f_\kappa(x_j)$, where $f_\kappa$ has degree $2\kappa+1$.
The slope of $f_\kappa$ at the origin is
\[
A_\kappa
\coloneqq
f_\kappa'(0)
=
(2\kappa+1)c_\kappa
=
\frac{(2\kappa+1)(2\kappa)!}
     {4^\kappa(\kappa!)^2}.
\]
For an integer $q\geq0$, define
\[
h_{\kappa,q}
\coloneqq
f_\kappa^{\circ q}
\quad\text{on }[0,1],
\qquad
h_{\kappa,q}(x)\coloneqq1
\quad\text{for }x\ge1.
\]
The map considered in~\cref{sec:ns} is recovered by setting $\kappa=2$, in which case
\[
f_2(x)
=
\frac{15}{8}x-\frac54x^3+\frac38x^5,
\qquad
A_2=\frac{15}{8},
\qquad
h_{2,q}=h_q.
\]
The following lemma generalizes~\cref{lem:scalar-ns} to every $\kappa\geq1$.
Thus, for every Taylor order $\kappa$, the penalty term in~\cref{eq:expected-discounted-regret} decreases at rate $A_\kappa^{-q}$, while the stability term grows at rate $A_\kappa^q$.
Consequently, the argument in~\cref{sec:ns} applies with $A_\kappa$ in place of $A$, and our analysis extends to the Newton--Schulz polynomials obtained by truncating the Taylor expansion at any finite order.

The following lemma establishes the bounds for every $\kappa\geq1$ and
yields~\cref{lem:scalar-ns} when $\kappa=2$.
\begin{lemma}[Extension of~\Cref{lem:scalar-ns}]
\label{lem:general-scalar-ns}
For every integer $\kappa\geq1$ and $q\geq0$,
\[
  \Lip(h_{\kappa,q})=A_\kappa^q,
  \qquad
  \frac{1}{2A_\kappa^q}
  \leq \Delta(h_{\kappa,q})
  \leq \frac{1}{A_\kappa^q+1}
  \leq \frac{1}{A_\kappa^q}.
\]
\end{lemma}

\begin{proof}
Fix an integer $\kappa\geq1$ and write $P_\kappa(u)\coloneqq\sum_{s=0}^{\kappa}c_su^s$.
Then $f_\kappa(x)=xP_\kappa(1-x^2)$.
With $u\coloneqq1-x^2$, differentiation and the identity $2(s+1)c_{s+1}=(2s+1)c_s$ give
\begin{align*}
  f_\kappa'(x)
  &=P_\kappa(u)-2(1-u)P_\kappa'(u) \\
  &=\sum_{s=0}^{\kappa}c_su^s-2(1-u)\sum_{s=1}^{\kappa}sc_su^{s-1} \\
  &=(2\kappa+1)c_\kappa u^\kappa
    +\sum_{s=0}^{\kappa-1}\bigl((2s+1)c_s-2(s+1)c_{s+1}\bigr)u^s \\
  &=(2\kappa+1)c_\kappa u^\kappa
  =A_\kappa(1-x^2)^\kappa.
\end{align*}
Therefore $0\leq f_\kappa'(x)\leq A_\kappa$ on $[0,1]$, with maximum $A_\kappa$ at $x=0$.
Since $f_\kappa'\geq0$, $f_\kappa(0)=0$, and $f_\kappa(1)=1$, the map $f_\kappa$ sends $[0,1]$ into itself, and hence so does every $h_{\kappa,q}$.
The case $q=0$ is immediate, and for $q\geq1$ the chain rule $h_{\kappa,q}'(x)=\prod_{j=0}^{q-1}f_\kappa'(h_{\kappa,j}(x))$ gives $0\leq h_{\kappa,q}'(x)\leq A_\kappa^q$ on $[0,1]$.
At $x=0$ every factor equals $f_\kappa'(0)=A_\kappa$ because $h_{\kappa,j}(0)=0$, and hence $h_{\kappa,q}'(0)=A_\kappa^q$.
Extending $h_{\kappa,q}$ by the constant $1$ on $[1,\infty)$ leaves this bound unchanged, so $\Lip(h_{\kappa,q})=A_\kappa^q$.

For the lower bound on $\Delta(h_{\kappa,q})$, the properties $h_{\kappa,q}(0)=0$, $\Lip(h_{\kappa,q})=A_\kappa^q$, and $h_{\kappa,q}\leq1$ give $h_{\kappa,q}(x)\leq\min\{A_\kappa^qx,1\}$.
Therefore
\[
  \Delta(h_{\kappa,q}) = \int_0^\infty (1 - h_{\kappa,q}(x))\,\mathrm{d}x
  \geq \int_0^{A_\kappa^{-q}}(1 - A_\kappa^q x)\,\mathrm{d}x
  =\frac{1}{A_\kappa^q} - \frac{1}{2A_\kappa^q}
  =\frac{1}{2A_\kappa^q}.
\]

For the upper bound on $\Delta(h_{\kappa,q})$, define for $L\geq1$ and $x\in[0,1]$
\[
  g_L(x)\coloneqq1-(1-x)^L.
\]
This family maps $[0,1]$ into itself and satisfies the composition identity $g_L\circ g_M=g_{LM}$ for $L,M\geq1$.
We first show that $f_\kappa(x)\geq g_{A_\kappa}(x)$ on $[0,1]$.
Both $f_\kappa'(x)=A_\kappa(1-x^2)^\kappa$ and $g_{A_\kappa}'(x)=A_\kappa(1-x)^{A_\kappa-1}$ are positive for $x\in[0,1)$, so $(f_\kappa-g_{A_\kappa})'(x)$ has the same sign as
\[
  \psi_\kappa(x)
  \coloneqq\log\frac{f_\kappa'(x)}{g_{A_\kappa}'(x)}
  =\kappa\log(1+x)+(\kappa+1-A_\kappa)\log(1-x),
\]
whose derivative is
\[
  \psi_\kappa'(x)
  =\frac{A_\kappa-1-(2\kappa+1-A_\kappa)x}{1-x^2}.
\]
To determine the sign of the numerator, we bound $A_\kappa$.
Since $f_\kappa(1)-f_\kappa(0)=1$, integrating $f_\kappa'(x)=A_\kappa(1-x^2)^\kappa$ over $[0,1]$ gives $A_\kappa^{-1}=\int_0^1(1-x^2)^\kappa\,\mathrm{d}x$.
Combining this with $(1-x)^\kappa<(1-x^2)^\kappa<1$ on $(0,1)$ gives
\[
  \frac{1}{\kappa+1}
  =\int_0^1(1-x)^\kappa\,\mathrm{d}x
  <\int_0^1(1-x^2)^\kappa\,\mathrm{d}x
  =\frac{1}{A_\kappa}
  <1,
\]
so $1<A_\kappa<\kappa+1$.
Hence the numerator is positive at $x=0$ and negative at $x=1$, and it has a unique zero in $(0,1)$.
The denominator is positive on $[0,1)$, so $\psi_\kappa$ first increases and then decreases.
Since $\psi_\kappa(0)=0$ and $\psi_\kappa(x)\to-\infty$ as $x\to1$, there is a unique $\bar x\in(0,1)$ such that
\[
  \psi_\kappa(x)\geq0\quad\text{for }x\in[0,\bar x],
  \qquad
  \psi_\kappa(x)\leq0\quad\text{for }x\in[\bar x,1).
\]
Hence $f_\kappa-g_{A_\kappa}$ is nondecreasing on $[0,\bar x]$ and nonincreasing on $[\bar x,1]$, and since $f_\kappa(0)-g_{A_\kappa}(0)=f_\kappa(1)-g_{A_\kappa}(1)=0$, we conclude $f_\kappa(x)\geq g_{A_\kappa}(x)$ on $[0,1]$.

We next prove $h_{\kappa,q}\geq g_{A_\kappa^q}$ for every $q\geq0$ by induction on $q$.
The base case $h_{\kappa,0}=g_1$ is immediate.
Assuming $h_{\kappa,q}\geq g_{A_\kappa^q}$, the monotonicity of $f_\kappa$, the bound $f_\kappa\geq g_{A_\kappa}$, and the composition identity give
\[
  h_{\kappa,q+1}
  =f_\kappa\circ h_{\kappa,q}
  \geq f_\kappa\circ g_{A_\kappa^q}
  \geq g_{A_\kappa}\circ g_{A_\kappa^q}
  =g_{A_\kappa^{q+1}},
\]
which completes the induction.
Finally, $h_{\kappa,q}\geq g_{A_\kappa^q}$ and $h_{\kappa,q}(x)=1$ for $x\geq1$ give
\[
  \Delta(h_{\kappa,q})
  =\int_0^1(1-h_{\kappa,q}(x))\,\mathrm{d}x
  \leq\int_0^1(1-x)^{A_\kappa^q}\,\mathrm{d}x
  =\frac{1}{A_\kappa^q+1}
  \leq\frac{1}{A_\kappa^q},
\]
which completes the proof.
\end{proof}

\subsection{Proof of~\Cref{thm:fixed-depth-discounted-regret}}

\begin{proof}[Proof of~\Cref{thm:fixed-depth-discounted-regret}]
By the definition of the depth in~\cref{thm:fixed-depth-discounted-regret},
\[
  \frac{G_{\mathrm{op}}\sqrt r}{\Gamma\sqrt{1-\beta}}\leq A^q<\frac{A G_{\mathrm{op}}\sqrt r}{\Gamma\sqrt{1-\beta}}.
\]
Then, by $C=G_{\mathrm{op}}$ together with $\Delta(h_q)\leq A^{-q}$ and $\Lip(h_q)=A^q$ from~\cref{lem:scalar-ns}, the first term in~\cref{eq:expected-discounted-regret} of~\cref{thm:general-discounted-regret} is bounded by
\[
  \frac{r G_{\mathrm{op}}}{(1-\beta)A^q}
  \leq\frac{r G_{\mathrm{op}}}{(1-\beta)\left(G_{\mathrm{op}}\sqrt r/(\Gamma\sqrt{1-\beta})\right)}
  =\Gamma\sqrt{\frac{r}{1-\beta}}.
\]
The second term satisfies
\[
  \frac{2A^q\Gamma^2}{G_{\mathrm{op}}}
  <\frac{2A\left(G_{\mathrm{op}}\sqrt r/(\Gamma\sqrt{1-\beta})\right)\Gamma^2}{G_{\mathrm{op}}}
  =2A\Gamma\sqrt{\frac{r}{1-\beta}}.
\]
Adding these two bounds yields~\cref{eq:fixed-depth-discounted-regret},
which completes the proof.
\end{proof}

\subsection{Proof of~\Cref{thm:ns-complexity}}
\begin{proof}[Proof of~\Cref{thm:ns-complexity}]
If $\epsilon>B_{\mathrm{NS}}$, the definition of
$B_{\mathrm{NS}}$ and $A=15/8$ give
$
  B_{\mathrm{NS}}
  \geq (1+2A)\sqrt r\,\Gamma
  > 3\sqrt r\,\Gamma.
$
Hence, every point $W$ satisfies
\[
  \lVert\nabla\mathcal L(W)\rVert_*^{[\rho]}
  \leq \lVert\nabla\mathcal L(W)\rVert_*
  \leq \sqrt r\,\lVert\nabla\mathcal L(W)\rVert_{\mathrm F}
  \leq \sqrt r\,\Gamma
  < \frac{\epsilon}{3}.
\]
Consequently,
$\E\brk{\lVert\nabla\mathcal L(\bar W_\tau)\rVert_*^{[\rho]}}
<\epsilon/3$, so~\cref{eq:stationarity-master-bound} holds in this case.

It remains to consider $\epsilon\leq B_{\mathrm{NS}}$.
Then $1-\beta=(\epsilon/3B_{\mathrm{NS}})^2\leq1/9$.
We first verify the condition of~\cref{thm:fixed-depth-discounted-regret}. $\Gamma\leq\sqrt r G_{\mathrm{op}}$ (from $\lVert G_s\rVert_{\mathrm F}\leq\sqrt r\lVert G_s\rVert_{\mathrm{op}}\leq\sqrt r G_{\mathrm{op}}$) gives $G_{\mathrm{op}}\sqrt r/(\Gamma\sqrt{1-\beta})\geq1/\sqrt{1-\beta}>1$, so, using $\sqrt{1-\beta}=\epsilon/(3B_{\mathrm{NS}})$, the depth specified in~\cref{thm:ns-complexity} equals the depth required in~\cref{thm:fixed-depth-discounted-regret}.
Hence, $\DReg_t(D)=D\DReg_t(1)$ gives
\[
  \E[\DReg_t(1)]\leq(1+2A)\Gamma\sqrt{\frac{r}{1-\beta}}.
\]
Substituting this into the regret term of~\cref{lem:jiang-o2nc} and using $\left(1+(1-\beta)(T-1)\right)/T=1-\beta+\beta/T$ gives
\[
  \frac{1}{T}\E\brk*{\DReg_T(1)+(1-\beta)\sum_{t=1}^{T-1}\DReg_t(1)}
  \leq\left(1-\beta+\frac{\beta}{T}\right)(1+2A)\Gamma\sqrt{\frac{r}{1-\beta}}.
\]
Since $\sqrt{1-\beta^2}\geq\sqrt{1-\beta}$, it holds that
\[
  \frac{1-\beta+\beta/T}{\sqrt{1-\beta^2}}
  \leq\frac{1-\beta+\beta/T}{\sqrt{1-\beta}}
  =\sqrt{1-\beta}+\frac{\beta}{T\sqrt{1-\beta}},
\]
and $B_{\mathrm{NS}}=\sqrt r\left((1+2A)\Gamma+\sigma\right)$ gives
\[
  \frac{1}{T}\E\brk*{\DReg_T(1)+(1-\beta)\sum_{t=1}^{T-1}\DReg_t(1)}
  +\left(1-\beta+\frac{\beta}{T}\right)\frac{\sqrt r\,\sigma}{\sqrt{1-\beta^2}}
  \leq B_{\mathrm{NS}}\left(\sqrt{1-\beta}+\frac{\beta}{T\sqrt{1-\beta}}\right).
\]
Thus~\cref{lem:jiang-o2nc} gives
\[
  \E\brk*{\lVert\nabla\mathcal L(\bar W_\tau)\rVert_*^{[\rho]}}
  \leq\frac{4\Delta_{\mathcal L}}{(1-\beta)\rho T}
  +B_{\mathrm{NS}}\left(\sqrt{1-\beta}+\frac{\beta}{T\sqrt{1-\beta}}\right).
\]
Here $B_{\mathrm{NS}}\sqrt{1-\beta}=\epsilon/3$, the first term is $36\Delta_{\mathcal L}B_{\mathrm{NS}}^2/(\rho\epsilon^2T)$, and the third term is $3\beta B_{\mathrm{NS}}^2/(\epsilon T)$.
Since $B_{\mathrm{NS}}^2=O(r(\Gamma+\sigma)^2)$, this proves~\cref{eq:stationarity-master-bound}, and requiring each of the two $T$-dependent terms to be at most $\epsilon/3$ gives~\cref{eq:iteration-complexity}, which completes the proof.
\end{proof}

\subsection{Deferred proofs from~\Cref{sec:ns-smooth}}
\label{app:smooth}

Here, we prove the results of~\cref{sec:ns-smooth} under~\cref{ass:oracle,ass:smooth}.
Write $g_s\coloneqq\nabla\mathcal L(\widetilde W_s)$ for the gradient at the query point of round $s$, and let $\mathcal Q_s$ be the $\sigma$-algebra of the history up to and including the choice of $\widetilde W_s$, so that $g_s$ and $M_{s-1}$ are $\mathcal Q_s$-measurable.
Thus, \cref{ass:oracle} with $W=\widetilde W_s$ gives $\E[G_s\mid\mathcal Q_s]=g_s$ and $\E[\lVert G_s-g_s\rVert_{\mathrm F}^2\mid\mathcal Q_s]\leq\sigma^2$.

\begin{proof}[Proof of~\cref{lem:smooth-bridge}]
For any $p\in\mathcal P(W;\rho)$, the triangle inequality and~\cref{eq:op-nuclear-smoothness} give
\[
  \lVert\nabla\mathcal L(W)\rVert_*
  \leq\lVert\E_{Y\sim p}\brk*{\nabla\mathcal L(Y)}\rVert_*
  +\E_{Y\sim p}\brk*{\lVert\nabla\mathcal L(W)-\nabla\mathcal L(Y)\rVert_*}
  \leq\lVert\E_{Y\sim p}\brk*{\nabla\mathcal L(Y)}\rVert_*+L_{\mathrm{op}}\rho ,
\]
using $\E_{Y\sim p}\brk*{\lVert Y-W\rVert_{\mathrm{op}}}\leq\rho$.
Taking the infimum over $p$ completes the proof.
\end{proof}

\begin{lemma}
\label{lem:bregman-split}
Under~\cref{ass:oracle}, for every $s\geq1$,
\begin{equation}
  \E\left[B_{\widetilde\Phi_{h_q,G_{\mathrm{op}}}}(M_s\Vert M_{s-1})\,\middle|\,\mathcal Q_s\right]
  \leq
  2(1-\beta)\lVert g_s-M_{s-1}\rVert_*
  +\frac{A^q(1-\beta)^2}{2G_{\mathrm{op}}}\sigma^2 .
  \label{eq:bregman-split}
\end{equation}
\end{lemma}

\begin{proof}
Let $\Psi\coloneqq\widetilde\Phi_{h_q,G_{\mathrm{op}}}$ and $\widehat M_s\coloneqq\beta M_{s-1}+(1-\beta)g_s$ for simplicity.
First, the update $M_s$ is related to $\widehat M_s$ by
\[
  M_s=\widehat M_s+(1-\beta)(G_s-g_s),
  \qquad
  \widehat M_s=\beta M_{s-1}+(1-\beta)g_s .
\]
By the definition of the Bregman divergence, the points $M_{s-1},\widehat M_s,M_s$ satisfy
\begin{align}
  B_\Psi(M_s\Vert M_{s-1})
  =&B_\Psi(\widehat M_s\Vert M_{s-1})
  +B_\Psi(M_s\Vert\widehat M_s)
  +\inpr{\nabla\Psi(\widehat M_s)-\nabla\Psi(M_{s-1}),M_s-\widehat M_s}.
  \label{eq:three-point}
\end{align}
The last inner product has zero conditional expectation, since $\E[M_s-\widehat M_s\mid\mathcal Q_s]=(1-\beta)\E[G_s-g_s\mid\mathcal Q_s]=0$.

For $B_\Psi(\widehat M_s\Vert M_{s-1})$, using $\widehat M_s-M_{s-1}=(1-\beta)(g_s-M_{s-1})$, we have
\begin{align*}
  B_\Psi(\widehat M_s\Vert M_{s-1})
  &=\Psi(\widehat M_s)-\Psi(M_{s-1})-\inpr{\nabla\Psi(M_{s-1}),\widehat M_s-M_{s-1}}\\
  &\leq\left|\Psi(\widehat M_s)-\Psi(M_{s-1})\right|
  +\left|\inpr{\nabla\Psi(M_{s-1}),\widehat M_s-M_{s-1}}\right|\\
  &\leq2\lVert\widehat M_s-M_{s-1}\rVert_*
  =2(1-\beta)\lVert g_s-M_{s-1}\rVert_* ,
\end{align*}
where the last inequality follows from operator--nuclear duality and $\lVert\nabla\Psi(M)\rVert_{\mathrm{op}}=\lVert\calH_{h_q}(M/G_{\mathrm{op}})\rVert_{\mathrm{op}}\leq1$, which holds since $h_q\in[0,1]$.

For $B_\Psi(M_s\Vert\widehat M_s)$, \cref{lem:bregman-upper} with $\Lip(h_q)=A^q$ gives
\[
  B_\Psi(M_s\Vert\widehat M_s)
  \leq\frac{A^q}{2G_{\mathrm{op}}}\lVert M_s-\widehat M_s\rVert_{\mathrm F}^2
  =\frac{A^q(1-\beta)^2}{2G_{\mathrm{op}}}\lVert G_s-g_s\rVert_{\mathrm F}^2 .
\]
Taking the conditional expectation of~\cref{eq:three-point} and using $\E[\lVert G_s-g_s\rVert_{\mathrm F}^2\mid\mathcal Q_s]\leq\sigma^2$ gives~\cref{eq:bregman-split}, which completes the proof.
\end{proof}

\begin{lemma}
\label{lem:smooth-tracking}
Under~\cref{ass:oracle,ass:smooth}, for every $s\geq1$,
\[
  \E\brk*{\lVert g_s-M_{s-1}\rVert_*}
  \leq
  \beta^{s-1}\lVert\nabla\mathcal L(W_0)\rVert_*
  +\frac{2L_{\mathrm{op}}D}{1-\beta}
  +\sigma\sqrt{\frac{r(1-\beta)}{1+\beta}} .
\]
\end{lemma}

\begin{proof}
For $s\geq2$, since $\widetilde W_s=W_{s-1}+u_sX_s$, $\widetilde W_{s-1}=W_{s-2}+u_{s-1}X_{s-1}$, and $W_{s-1}=W_{s-2}+X_{s-1}$ with $u_s\in[0,1]$ in~\cref{alg:generic-o2nc}, $\widetilde W_s-\widetilde W_{s-1}=(1-u_{s-1})X_{s-1}+u_sX_s$.
Since $\lVert X_t\rVert_{\mathrm{op}}\leq D$, this and~\cref{ass:smooth} give 
\begin{equation}
  \lVert\widetilde W_s-\widetilde W_{s-1}\rVert_{\mathrm{op}}\leq2D,
  \qquad
  \lVert g_s-g_{s-1}\rVert_*\leq2L_{\mathrm{op}}D.
  \label{eq:query-drift}
\end{equation}
The definition of $M_{s-1}$ and $G_j=g_j+(G_j-g_j)$ give
\begin{equation}
  g_s -  M_{s-1}
  =\left(g_s - (1-\beta)\sum_{j=1}^{s-1}\beta^{s-1-j}g_j\right)
  -
  (1-\beta)\sum_{j=1}^{s-1}\beta^{s-1-j}(G_j-g_j).
  \label{eq:ema-signal-noise}
\end{equation}
The first term of~\cref{eq:ema-signal-noise} can be rewritten as
\[
  g_s-(1-\beta)\sum_{j=1}^{s-1}\beta^{s-1-j}g_j
  =\beta^{s-1}g_1
  +\sum_{j=2}^s\beta^{s-j}(g_j-g_{j-1}).
\]
Since $X_1=0$ we have $\widetilde W_1=W_0$ and $g_1=\nabla\mathcal L(W_0)$, so~\cref{eq:query-drift} yields
\begin{align}
  \left\lVert g_s-(1-\beta)\sum_{j=1}^{s-1}\beta^{s-1-j}g_j\right\rVert_*
  &\leq\beta^{s-1}\lVert\nabla\mathcal L(W_0)\rVert_*
  +2L_{\mathrm{op}}D\sum_{j=2}^s\beta^{s-j}
  \notag\\
  &\leq\beta^{s-1}\lVert\nabla\mathcal L(W_0)\rVert_*
  +\frac{2L_{\mathrm{op}}D}{1-\beta}.
  \label{eq:signal-track}
\end{align}
For the second term of~\cref{eq:ema-signal-noise}, $\E[\lVert G_s-g_s\rVert_{\mathrm F}^2\mid\mathcal Q_s]\leq\sigma^2$
and $\E[G_s - g_s\mid\mathcal Q_s] = 0$ give
\begin{align*}
  \E\brk*{\left\lVert(1-\beta)\sum_{j=1}^{s-1}\beta^{s-1-j}(G_j-g_j)\right\rVert_{\mathrm F}^2}
  &=(1-\beta)^2\sum_{j=1}^{s-1}\beta^{2(s-1-j)}\E\brk*{\lVert G_j-g_j\rVert_{\mathrm F}^2}\\
  &\leq\frac{1-\beta}{1+\beta}\sigma^2 .
\end{align*}
Thus, Jensen's inequality and $\lVert Z\rVert_*\leq\sqrt r\lVert Z\rVert_{\mathrm F}$ give
\begin{equation}
  \E\brk*{\left\lVert(1-\beta)\sum_{j=1}^{s-1}\beta^{s-1-j}(G_j-g_j)\right\rVert_*}
  \leq\sigma\sqrt{\frac{r(1-\beta)}{1+\beta}} .
  \label{eq:noise-ema}
\end{equation}
Combining~\cref{eq:ema-signal-noise}, \cref{eq:signal-track}, and~\cref{eq:noise-ema} completes the proof.
\end{proof}

\begin{proof}[Proof of~\Cref{thm:smooth-stationarity}]
If $\epsilon>\frac{3}{2}\sqrt r\,\Gamma$, then Jensen's inequality and
\cref{eq:oracle-bounds} of~\cref{ass:oracle} imply that every point $W$ satisfies
\[
  \lVert\nabla\mathcal L(W)\rVert_*
  \leq \sqrt r\,\lVert\nabla\mathcal L(W)\rVert_{\mathrm F}
  \leq \sqrt r\,\Gamma
  < \frac{2}{3} \epsilon.
\]
Consequently,~\cref{eq:smooth-stationarity} holds in this case.

It remains to consider $\epsilon\leq\frac{3}{2}\sqrt r\,\Gamma$.
Let $\Psi\coloneqq\widetilde\Phi_{h_q,G_{\mathrm{op}}}$ for simplicity.
The choice $C_s = G_{\mathrm{op}}$ satisfies
$C_{s+1}\geq\beta C_s$, so~\cref{eq:bregman-innovation} of~\cref{thm:general-discounted-regret} applies with
$h=h_q$.
Using $\Delta(h_q)\leq A^{-q}$ from \cref{lem:scalar-ns} and
$\DReg_t(D)=D\DReg_t(1)$, dividing by $D$ and taking expectations gives
\begin{align*}
  \E[\DReg_t(1)]
  &\leq
  \frac{rG_{\mathrm{op}}}{(1-\beta)A^q}
  +\frac{1}{1-\beta}\sum_{s=1}^t\beta^{t-s}
    \E\brk*{B_\Psi(M_s\Vert M_{s-1})}\\
  &\leq
  \frac{rG_{\mathrm{op}}}{(1-\beta)A^q}
  +\frac{A^q(1-\beta)\sigma^2}{2G_{\mathrm{op}}}
    \sum_{s=1}^t\beta^{t-s}
  +2\sum_{s=1}^t\beta^{t-s}
    \E\brk*{\lVert g_s-M_{s-1}\rVert_*}\\
  &\leq
  \frac{rG_{\mathrm{op}}}{(1-\beta)A^q}
  +\frac{A^q(1-\beta)\sigma^2}{2G_{\mathrm{op}}}
    \sum_{s=1}^t\beta^{t-s}\\
  &\qquad
  +2\sum_{s=1}^t\beta^{t-s}
    \left(
      \beta^{s-1}\lVert\nabla\mathcal L(W_0)\rVert_*
      +\frac{2L_{\mathrm{op}}D}{1-\beta}
      +\sigma\sqrt{\frac{r(1-\beta)}{1+\beta}}
    \right),
\end{align*}
where in the second inequality we used \cref{lem:bregman-split} and the tower property, 
and in the last inequality we used \cref{lem:smooth-tracking}.
Using $\sum_{s=1}^t\beta^{t-s}\leq1/(1-\beta)$ and
$\sum_{s=1}^t\beta^{t-s}\beta^{s-1}=t\beta^{t-1}$, we further bound this as
\begin{align}
  \E[\DReg_t(1)]
  &\leq
  \frac{rG_{\mathrm{op}}}{(1-\beta)A^q}
  +\frac{A^q\sigma^2}{2G_{\mathrm{op}}}
  +2t\beta^{t-1}\lVert\nabla\mathcal L(W_0)\rVert_*
  +\frac{4L_{\mathrm{op}}D}{(1-\beta)^2}
  +\frac{2\sqrt r\,\sigma}{\sqrt{1-\beta^2}}.
  \label{eq:smooth-regret-terminal}
\end{align}

The term $2t\beta^{t-1}\lVert\nabla\mathcal L(W_0)\rVert_*$ contributes to the bound of~\cref{lem:jiang-o2nc} as 
\begin{align*}
  &\frac1T\left(2T\beta^{T-1}+2(1-\beta)\sum_{t=1}^{T-1}t\beta^{t-1}\right)\lVert\nabla\mathcal L(W_0)\rVert_*\\
  &\leq\frac1T\left(\frac{2}{1-\beta}+\frac{2}{1-\beta}\right)\lVert\nabla\mathcal L(W_0)\rVert_*
  =\frac{4\lVert\nabla\mathcal L(W_0)\rVert_*}{T(1-\beta)},
\end{align*}
where the inequality uses $T(1-\beta)\beta^{T-1}\leq(1-\beta)\sum_{t=1}^{T}\beta^{t-1}=1-\beta^{T}\leq1$ and $\sum_{t\geq1}t\beta^{t-1}\leq(1-\beta)^{-2}$.

Set $\rho\coloneqq\epsilon/(12L_{\mathrm{op}})$, so that $D=(1-\beta)\rho/(4\beta)$.
Substituting~\cref{eq:smooth-regret-terminal} into~\cref{lem:jiang-o2nc} and using $(1+(1-\beta)(T-1))/T=1-\beta+\beta/T$ gives
\begin{align*}
  \E\brk*{\lVert\nabla\mathcal L(\bar W_\tau)\rVert_*^{[\rho]}}
  &\leq
  \frac{4\Delta_{\mathcal L}}{(1-\beta)\rho T}
  +\frac{4\lVert\nabla\mathcal L(W_0)\rVert_*}{T(1-\beta)}\\
  &\qquad
  +\left(1-\beta+\frac{\beta}{T}\right)\left(
    \frac{rG_{\mathrm{op}}}{(1-\beta)A^q}
    +\frac{4L_{\mathrm{op}}D}{(1-\beta)^2}
    +\frac{3\sqrt r\,\sigma}{\sqrt{1-\beta^2}}
    +\frac{A^q\sigma^2}{2G_{\mathrm{op}}}
  \right).
\end{align*}
Then, using~\cref{lem:smooth-bridge} and substituting $D=(1-\beta)\rho/(4\beta)$ gives
\begin{align*}
  \E\brk*{\lVert\nabla\mathcal L(\bar W_\tau)\rVert_*}
  &\leq
  \frac{4\Delta_{\mathcal L}}{(1-\beta)\rho T}
  +\frac{4\lVert\nabla\mathcal L(W_0)\rVert_*}{T(1-\beta)}
  +\left(1+\frac{\beta}{T(1-\beta)}\right)\frac{rG_{\mathrm{op}}}{A^q}\\
  &\qquad
  +L_{\mathrm{op}}\rho\left(1+\frac1\beta+\frac{1}{T(1-\beta)}\right)
  +3\sqrt r\sigma\frac{1-\beta+\beta/T}{\sqrt{1-\beta^2}}
  +\left(1-\beta+\frac{\beta}{T}\right)\frac{A^q\sigma^2}{2G_{\mathrm{op}}} .
\end{align*}
We next bound the six terms one by one.
The choices of $\rho$ and $\beta$ give $L_{\mathrm{op}}\rho=\epsilon/12$ and $1/(1-\beta)=\max\{2,144r\sigma^2/\epsilon^2\}\leq2+144r\sigma^2/\epsilon^2$.
Thus, the first term satisfies
\begin{align*}
  \frac{4\Delta_{\mathcal L}}{(1-\beta)\rho T}
  &=O\left(
    \frac{L_{\mathrm{op}}\Delta_{\mathcal L}}{\epsilon T}
    +\frac{r\sigma^2L_{\mathrm{op}}\Delta_{\mathcal L}}{\epsilon^3T}
  \right).
\end{align*}

For the second term, we first bound $\lVert\nabla\mathcal L(W_0)\rVert_*$.
Applying the fundamental theorem of calculus along the segment from $W$ to $W'$, followed by operator--nuclear duality and~\cref{ass:smooth}, gives
\begin{equation}
  \mathcal L(W')\leq\mathcal L(W)+\inpr{\nabla\mathcal L(W),W'-W}+\frac{L_{\mathrm{op}}}{2}\lVert W'-W\rVert_{\mathrm{op}}^2.
  \label{eq:descent-ineq}
\end{equation}
Take $W=W_0$ and $W'=W_0-\lVert\nabla\mathcal L(W_0)\rVert_*\polar(\nabla\mathcal L(W_0))/L_{\mathrm{op}}$.
Then $\inpr{\nabla\mathcal L(W_0),W'-W_0}=-\lVert\nabla\mathcal L(W_0)\rVert_*^2/L_{\mathrm{op}}$ and $\lVert W'-W_0\rVert_{\mathrm{op}}=\lVert\nabla\mathcal L(W_0)\rVert_*/L_{\mathrm{op}}$, so \cref{eq:descent-ineq} gives $\mathcal L(W')\leq\mathcal L(W_0)-\lVert\nabla\mathcal L(W_0)\rVert_*^2/(2L_{\mathrm{op}})$.
Combining this with $\mathcal L(W')\geq\mathcal L(W_0)-\Delta_{\mathcal L}$ yields $\lVert\nabla\mathcal L(W_0)\rVert_*\leq\sqrt{2L_{\mathrm{op}}\Delta_{\mathcal L}}\leq\epsilon+L_{\mathrm{op}}\Delta_{\mathcal L}/\epsilon$ and then
\begin{align*}
  \frac{4\lVert\nabla\mathcal L(W_0)\rVert_*}{(1-\beta)T}
  &\leq\frac{4\sqrt{2L_{\mathrm{op}}\Delta_{\mathcal L}}}{T}\Big(2+\frac{144r\sigma^2}{\epsilon^2}\Big)
  =O\left(
    \frac{\epsilon}{T}
    +\frac{L_{\mathrm{op}}\Delta_{\mathcal L}}{\epsilon T}
    +\frac{r\sigma^2}{\epsilon T}
    +\frac{r\sigma^2L_{\mathrm{op}}\Delta_{\mathcal L}}{\epsilon^3T}
  \right).
\end{align*}

For the third term, $\Gamma\leq\sqrt r\,G_{\mathrm{op}}$ and
$\epsilon\leq\frac{3}{2}\sqrt r\,\Gamma$ imply
$12rG_{\mathrm{op}}/\epsilon\geq8>1$.
Thus, $q=\lceil\log_A(12rG_{\mathrm{op}}/\epsilon)\rceil$ and
$12rG_{\mathrm{op}}/\epsilon\leq A^q<12ArG_{\mathrm{op}}/\epsilon$.
It follows that
\begin{align*}
  \Big(1+\frac{\beta}{T(1-\beta)}\Big)\frac{rG_{\mathrm{op}}}{A^q}
  &\leq\frac{\epsilon}{12}+\frac{\epsilon}{12T}\Big(2+\frac{144r\sigma^2}{\epsilon^2}\Big)
  =\frac{\epsilon}{12}+O\left(\frac{\epsilon}{T}+\frac{r\sigma^2}{\epsilon T}\right).
\end{align*}

For the fourth term, $1-\beta\leq1/2$ implies $1/\beta\leq2$.
Together with $L_{\mathrm{op}}\rho=\epsilon/12$, this gives
\begin{align*}
  L_{\mathrm{op}}\rho\Big(1+\frac1\beta+\frac{1}{T(1-\beta)}\Big)
  &\leq\frac{\epsilon}{4}+\frac{\epsilon}{12T}\Big(2+\frac{144r\sigma^2}{\epsilon^2}\Big)
  \leq\frac{\epsilon}{4}+O\left(\frac{\epsilon}{T}+\frac{r\sigma^2}{\epsilon T}\right).
\end{align*}

For the fifth term, the definition of $\beta$ ensures $3\sqrt r\sigma\sqrt{1-\beta}\leq\epsilon/4$.
Moreover, $\sqrt{1-\beta^2}\geq\sqrt{1-\beta}$, $1/\sqrt{1-\beta}\leq\sqrt2+12\sqrt r\sigma/\epsilon$, and $\sqrt{r\sigma^2}\leq(\epsilon+r\sigma^2/\epsilon)/2$ give
\begin{align*}
  3\sqrt r\sigma\frac{1-\beta+\beta/T}{\sqrt{1-\beta^2}}
  &\leq3\sqrt r\sigma\sqrt{1-\beta}
  +\frac{3\sqrt r\sigma}{T\sqrt{1-\beta}}
  \leq\frac{\epsilon}{4}+O\left(\frac{\epsilon}{T}+\frac{r\sigma^2}{\epsilon T}\right).
\end{align*}

For the sixth term, $(1-\beta)\sigma^2\leq\epsilon^2/(144r)$ and the upper bound on $A^q$ above give
\begin{align*}
  \Big(1-\beta+\frac{\beta}{T}\Big)\frac{A^q\sigma^2}{2G_{\mathrm{op}}}
  &\leq\frac{\epsilon}{12}+O\left(\frac{r\sigma^2}{\epsilon T}\right).
\end{align*}
Therefore,
\begin{align*}
  \E\brk*{\lVert\nabla\mathcal L(\bar W_\tau)\rVert_*}
  &\leq
  \frac{2}{3}\epsilon + O\left(
    \frac{\epsilon}{T}
    +\frac{L_{\mathrm{op}}\Delta_{\mathcal L}}{\epsilon T}
    +\frac{r\sigma^2}{\epsilon T}
    +\frac{r\sigma^2L_{\mathrm{op}}\Delta_{\mathcal L}}{\epsilon^3T}
  \right).
\end{align*}
This gives~\cref{eq:smooth-stationarity}, and requiring the $1/T$ terms to be at most $\epsilon/3$ gives~\cref{eq:smooth-complexity}, which completes the proof.
\end{proof}

\section{Deferred details on the FTRL interpretation (\Cref{sec:ftrl})}
\label{app:ftrl}

For the FTRL interpretation, we compute the Fenchel conjugate of the smoothed potential $\widetilde\Phi_{h,C}$ and the regularizer it induces, and then specialize the results to the Newton--Schulz map $h_q$.
The regularizer is a sum of the scalar conjugate over the singular values.

\subsection{General spectral map $h$ (proof of~\Cref{thm:general-ftrl})}
\label{sec:general-ftrl}
\begin{lemma}
\label{lem:general-scalar-conjugate}
Extend $\phi_h$ to a proper convex function on $\R$ by setting $\phi_h(x)=+\infty$ for $x<0$, and let $\phi_h^*(a)\coloneqq\sup_{x\geq0}\{ax-\phi_h(x)\}$ be its Fenchel conjugate.
Under~\cref{ass:general-spectral-link}, $\phi_h^*(a)<\infty$ if and only if $a\leq1$.
\end{lemma}

\begin{proof}
For $a\leq1$ and $x\geq0$,
\[
  ax-\phi_h(x)
  \leq x-\phi_h(x)
  =\int_0^x(1-h(u))\,\mathrm{d}u
  \leq
  \int_0^\infty (1-h(u))\,\mathrm{d}u
  =
  \Delta(h),
\]
so $\phi_h^*(a)\leq\Delta(h)<\infty$.
For $a>1$, the bound $\phi_h(x)\leq x$ gives $ax-\phi_h(x)\geq(a-1)x\to\infty$ as $x\to\infty$, so $\phi_h^*(a)=+\infty$, which completes the proof.
\end{proof}

Lifting the scalar conjugate to matrices via~\cref{lem:aux-conjugacy} gives the Fenchel conjugate of the smoothed potential.

\begin{lemma}
\label{lem:spectral-conjugate}
Under~\cref{ass:general-spectral-link}, for every $C>0$, the Fenchel conjugate of $\widetilde\Phi_{h,C}$ is
\begin{equation}
  \widetilde\Phi_{h,C}^*(W)
  =C\sum_{i=1}^r\phi_h^*(\sigma_i(W)).
  \notag
\end{equation}
\end{lemma}

\begin{proof}
Apply the Fenchel conjugate formula for singular-value functions~(\cref{lem:aux-conjugacy}) to $\phi(x)\coloneqq C\phi_h(x/C)$.
This expresses the Fenchel conjugate of $\widetilde\Phi_{h,C}$ as $W\mapsto\sum_{i=1}^r\phi^*(\sigma_i(W))$.
The change of variables $x=Cu$ gives $\phi^*(a)=\sup_{x\geq0}\{ax-C\phi_h(x/C)\}=C\phi_h^*(a)$, which completes the proof.
\end{proof}

We are now ready to prove~\Cref{thm:general-ftrl}.

\begin{proof}[Proof of~\Cref{thm:general-ftrl}]
By~\cref{eq:general-gbpa-action}, it holds that
\[
  X_t^h
  =-D\calH_h\left(\frac{M_{t-1}}{C_t}\right)
  =-D\nabla\widetilde\Phi_{h,C_t}(M_{t-1}).
\]
It remains to show that $X_t^h$ minimizes the FTRL objective.
The definition of $M_{t-1}$ gives $M_{t-1}=(1-\beta)\beta^{t-1}\sum_{s=1}^{t-1}\beta^{-s}G_s$.
Hence the objective in~\cref{eq:general-potential-action} is
$\big(\inpr{M_{t-1},X}+R_{h,C_t}(X)\big)/\big((1-\beta)\beta^{t-1}\big)$.
Substituting $X=-DW$ and using $R_{h,C_t}(X)=D\widetilde\Phi_{h,C_t}^*(X/D)$ together with the evenness of $\widetilde\Phi_{h,C_t}^*$ (which holds since $\sigma_i(-W)=\sigma_i(W)$) gives
\[
  \inpr{M_{t-1},X}+R_{h,C_t}(X)
  =-D
  \left(\inpr{M_{t-1},W}-\widetilde\Phi_{h,C_t}^*(W)\right).
\]
Thus minimizing the left-hand side over $\lVert X\rVert_{\mathrm{op}}\leq D$ is the same as maximizing $\inpr{M_{t-1},W}-\widetilde\Phi_{h,C_t}^*(W)$ over $\lVert W\rVert_{\mathrm{op}}\leq1$.
By the Fenchel--Young inequality,
\[
  \inpr{M_{t-1},W}-\widetilde\Phi_{h,C_t}^*(W)
  \leq\widetilde\Phi_{h,C_t}(M_{t-1}),
\]
with equality if and only if $W=\nabla\widetilde\Phi_{h,C_t}(M_{t-1})$.
The minimizer is therefore $X=-D\nabla\widetilde\Phi_{h,C_t}(M_{t-1})=X_t^h$, which completes the proof.
\end{proof}

\paragraph{Closed form of the regularizer.}

The regularizer~\cref{eq:general-ftrl-regularizer} is a sum of the scalar conjugate $\phi_h^*$ over the singular values.
In terms of the generalized inverse $h^{-1}(v)\coloneqq\inf\{x\geq0:h(x)\geq v\}$ with $\inf\emptyset\coloneqq+\infty$, the scalar conjugate is
\begin{equation}
  \phi_h^*(a)=
  \begin{cases}
    0, & a\leq0,\\[1mm]
    \displaystyle\int_0^a h^{-1}(v)\,\mathrm{d}v, & 0\leq a\leq1,\\[3mm]
    +\infty, & a>1.
  \end{cases}
  \label{eq:scalar-conjugate}
\end{equation}
Indeed, for $a\leq0$ the map $x\mapsto ax-\phi_h(x)$ is nonincreasing, so $\phi_h^*(a)=0$, and \cref{lem:general-scalar-conjugate} gives $\phi_h^*(a)=+\infty$ for $a>1$.
For $a\in[0,1]$, the integrand of $ax-\phi_h(x)=\int_0^x(a-h(u))\,\mathrm{d}u$ is positive exactly when $u<h^{-1}(a)$, so
\[
  \phi_h^*(a)
  =\int_0^{h^{-1}(a)}(a-h(u))\,\mathrm{d}u
  =\int_0^a h^{-1}(v)\,\mathrm{d}v,
\]
where the second equality exchanges the order of integration over $\{(u,v):h(u)<v\leq a\}$.
Taking $a=1$ gives
\[
  \phi_h^*(1)
  =\int_0^{h^{-1}(1)}(1-h(u))\,\mathrm{d}u
  =\int_0^\infty(1-h(u))\,\mathrm{d}u
  =\Delta(h),
\]
where the second equality uses $1-h(u)=0$ for $u>h^{-1}(1)$.
Thus, since $\phi_h^*$ is nondecreasing, $0\leq\phi_h^*(a)\leq\Delta(h)$ for every $a\in[0,1]$.

\begin{figure}[t]
  \centering
  \includegraphics[width=0.53\linewidth]{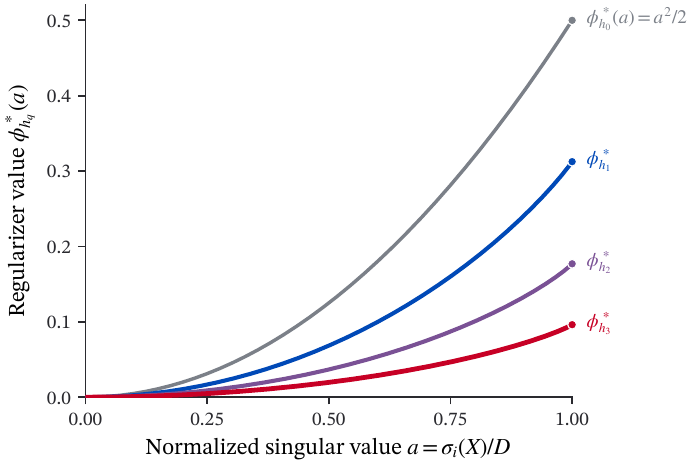}
  \caption{The scalar contribution $\phi_{h_q}^*(a)$ to the FTRL regularizer for depths $q=0,\ldots,3$, where $a=\sigma_i(X)/D\in[0,1]$.
  At depth zero, $\phi_{h_0}^*(a)=a^2/2$.
  The curves are pointwise nonincreasing in $q$, each marked endpoint equals $\phi_{h_q}^*(1)=\Delta(h_q)$, and these endpoint values converge to zero as $q\to\infty$.}
  \label{fig:ns-regularizer}
\end{figure}

\subsection{Specialization to Newton--Schulz $h_q$}
\label{app:ns-ftrl-details}

Write $Y=M_t/G_{\mathrm{op}}=U\diag(x_i)V^\top$ for a thin singular value decomposition, with $x_i\in[0,1]$ because $\lVert Y\rVert_{\mathrm{op}}\leq1$.
Then $(YY^\top)^kY=U\diag(x_i^{2k+1})V^\top$.
Using this for $k=0,1,2$, we see that the Newton--Schulz step of Line~\ref{algline:ns-step} acts as
\[
  \frac{15}{8}Y-\frac54(YY^\top)Y+\frac38(YY^\top)^2Y
  =U\diag\left(f(x_i)\right)V^\top.
\]
Since $f$ maps $[0,1]$ into itself, the singular values stay in $[0,1]$, so $q$ steps replace each $x_i$ by $f^{\circ q}(x_i)=h_q(x_i)$ from~\cref{eq:hq-definition}, and therefore the $q$-step iteration equals the singular-value map $\calH_{h_q}(Y)$ of~\cref{eq:ns-spectral-map}, as stated in~\cref{sec:muon}.

For the constant normalization $C_t=G_{\mathrm{op}}$, \cref{thm:general-ftrl} gives
\[
  X_t
  =-D\nabla\widetilde\Phi_{h_q,G_{\mathrm{op}}}(M_{t-1})
  \in\argmin_{\lVert X\rVert_{\mathrm{op}}\leq D}
  \left\{
    \inpr*{\sum_{s=1}^{t-1}\beta^{-s}G_s,X}
    +\frac{R_{h_q,G_{\mathrm{op}}}(X)}{(1-\beta)\beta^{t-1}}
  \right\}.
\]
Thus the update that applies $q$ Newton--Schulz steps to the momentum is the FTRL update for the discounted linear losses.

\paragraph{Vanishing regularizer as $q\to\infty$.}

By the closed form~\cref{eq:scalar-conjugate}, $\phi_{h_q}^*$ is nondecreasing with $0\leq\phi_{h_q}^*(a)\leq\phi_{h_q}^*(1)=\Delta(h_q)$ on $[0,1]$.
Since $\Delta(h_q)\to0$ as $q\to\infty$ by~\cref{eq:delta-q-bounds} of~\cref{lem:scalar-ns}, the FTRL update approaches the follow-the-leader action $-D\polar(M_{t-1})$ of exact-polar Muon.
\Cref{fig:ns-regularizer} shows $\phi_{h_q}^*$ for depths $q=0,\ldots,3$.

\section{Guarantees for other spectral maps}
\label{app:other-spectral-maps}

\Cref{sec:ns} focuses on the finite Newton--Schulz map, while the analysis of~\Cref{sec:general} applies to any spectral map satisfying~\cref{ass:general-spectral-link}.
We illustrate this generality using a recently proposed smooth relaxation of the polar map~\citep{mustafi2026move,feoktistov26softsign} as an example.
The relaxation inherits the stationarity guarantees for nonsmooth nonconvex objectives from~\cref{sec:o2nc,sec:general}, where the bound balances the approximation error $\Delta(h)$ against the Lipschitz constant $\Lip(h)$ as for finite Newton--Schulz.
A smaller $\Delta(h)$ forces a larger $\Lip(h)$, and we show that this tradeoff is unavoidable for every map satisfying~\cref{ass:general-spectral-link}.

\subsection{A smooth relaxation of the polar map}
For a smoothing parameter $a>0$, define
\[
  \widetilde h_a(x)
  \coloneqq
  \frac{a x}{\sqrt{1+a^2x^2}},
  \qquad x\in[0,\infty).
\]
This smooth approximation of the polar map is used in recent work~\citep{feoktistov26softsign,mustafi2026move}.
We first check that $\widetilde h_a$ satisfies~\cref{ass:general-spectral-link}.
It is continuous with $\widetilde h_a(0)=0$, and its derivative
\[
  \widetilde h_a'(x)
  =\frac{a}{\left(1+a^2x^2\right)^{3/2}}
\]
is positive, so $\widetilde h_a$ is strictly increasing, with $0\le \widetilde h_a(x)<1$ on $[0,\infty)$.
Since $\widetilde h_a'$ is decreasing on $[0,\infty)$, the Lipschitz constant is
\[
  \Lip(\widetilde h_a)=\widetilde h_a'(0)=a .
\]

The approximation error is
\[
  \Delta(\widetilde h_a)
  =\int_0^\infty
  \left(
  1-\frac{a x}{\sqrt{1+a^2x^2}}
  \right)\mathrm{d}x
  =\lim_{z\to\infty}\left[\,x-\frac1a\sqrt{1+a^2x^2}\,\right]_0^{z}
  =\frac1a ,
\]
where we used $x-\tfrac1a\sqrt{1+a^2x^2}=-\tfrac{1}{a^2x+a\sqrt{1+a^2x^2}}\to0$ as $x\to\infty$.
Combining these gives
\[
  \Lip(\widetilde h_a)=a,
  \qquad
  \Delta(\widetilde h_a)=\frac1a=\frac{1}{\Lip(\widetilde h_a)},
\]
which is the analogue of~\cref{lem:scalar-ns} for $\widetilde h_a$.
Balancing the approximation error against the Lipschitz constant as in~\cref{sec:ns} then yields guarantees such as~\cref{thm:fixed-depth-discounted-regret,thm:ns-complexity} for $\widetilde h_a$, now controlled by the smoothing parameter $a$ rather than the iteration depth $q$.

\subsection{The tradeoff is intrinsic to general spectral maps}

This inverse relationship between $\Delta(h)$ and $\Lip(h)$ is not special to the two maps above.
For any $h$ satisfying~\cref{ass:general-spectral-link}, the properties $h(0)=0$, $h\le1$, and $\Lip(h)$-Lipschitz continuity give $h(x)\le\min\{\Lip(h)\,x,\,1\}$, and therefore
\[
  \Delta(h)
  =\int_0^\infty\left(1-h(x)\right)\mathrm{d}x
  \ge\int_0^{1/\Lip(h)}\left(1-\Lip(h)\,x\right)\mathrm{d}x
  =\frac{1}{2\,\Lip(h)} .
\]
The approximation error therefore cannot fall below $1/(2\Lip(h))$, so $\Delta(h)$ and $\Lip(h)$ cannot both be made small.
The tradeoff between approximation and stability exploited by finite Newton--Schulz is thus intrinsic to every map satisfying~\cref{ass:general-spectral-link}, not an artifact of the Newton--Schulz iteration.
\section{Numerical experiments}
\label{app:numerical-experiments}

We test the finite Newton--Schulz update analyzed in~\cref{sec:ns} on a synthetic nonsmooth nonconvex objective.
By~\cref{lem:scalar-ns}, increasing the Newton--Schulz depth $q$ reduces the approximation error $\Delta(h_q)$ relative to the exact polar map but increases the Lipschitz constant $\Lip(h_q)$, and the regret bound of~\cref{thm:fixed-depth-discounted-regret} balances these two effects.
The purpose of the experiment is not to benchmark the optimization performance of the individual methods, but to examine qualitatively whether this depth-dependent penalty--stability tradeoff appears in the optimization behavior.
For this purpose, all methods evaluate the gradient at the current iterate rather than at the randomized query points of the O2NC conversion~(\cref{sec:o2nc}).

For $k=1,\ldots,N$, let $Z_k\in\R^{d\times d}$ be a standard Gaussian random matrix and set $J_k=SZ_kS$ with $S=\diag(s_1,\ldots,s_d)$ and $s_i=10^{-(i-1)/(d-1)}$.
For $W\in\R^{d\times d}$, we minimize an anisotropic variant of the objective employed in~\citet{jiang26adaptive},
\[
  \mathcal L(W)
  =\frac1N\sum_{k=1}^N\psi(\inpr{J_k,W}),
  \qquad
  \psi(z)=|z|\bigl(1-a\cos(\omega z)\bigr),
\]
with $d=20$, $N=100$, $a=0.9$, and $\omega=3$, and we select the subgradient $0$ at $z=0$.

All methods run on the same problem instance with learning rates $0.01$ and $0.05$ and momentum parameter $\beta=0.9$.
We compare the following methods:
\begin{itemize}
  \item Muon (SVD): the exact polar factor of the momentum matrix, computed by a singular value decomposition.
  \item Muon (NS): the update $\calH_{h_q}(M_t/G_{\mathrm{op}})$ with depth $q\in\{0,2,5,10\}$, where the learning rate plays the role of the radius $D$ in~\cref{sec:ns}, and $q=0$ uses the normalized momentum itself.
  We fix $G_{\mathrm{op}}=3.2$, which exceeds every gradient operator norm observed in the runs.
  \item Pion and Leon: Algorithms~2 and~3 of~\citet{jiang26adaptive} with step size $\eta=1$ in their notation, second-moment parameter $0.9$, and a numerical regularization of $10^{-8}$ in the preconditioner.
  Pion averages $10$ perturbation samples per step.
\end{itemize}

\begin{figure}[ht]
  \centering
  \begin{minipage}[t]{0.49\linewidth}
    \centering
    \includegraphics[width=\linewidth]{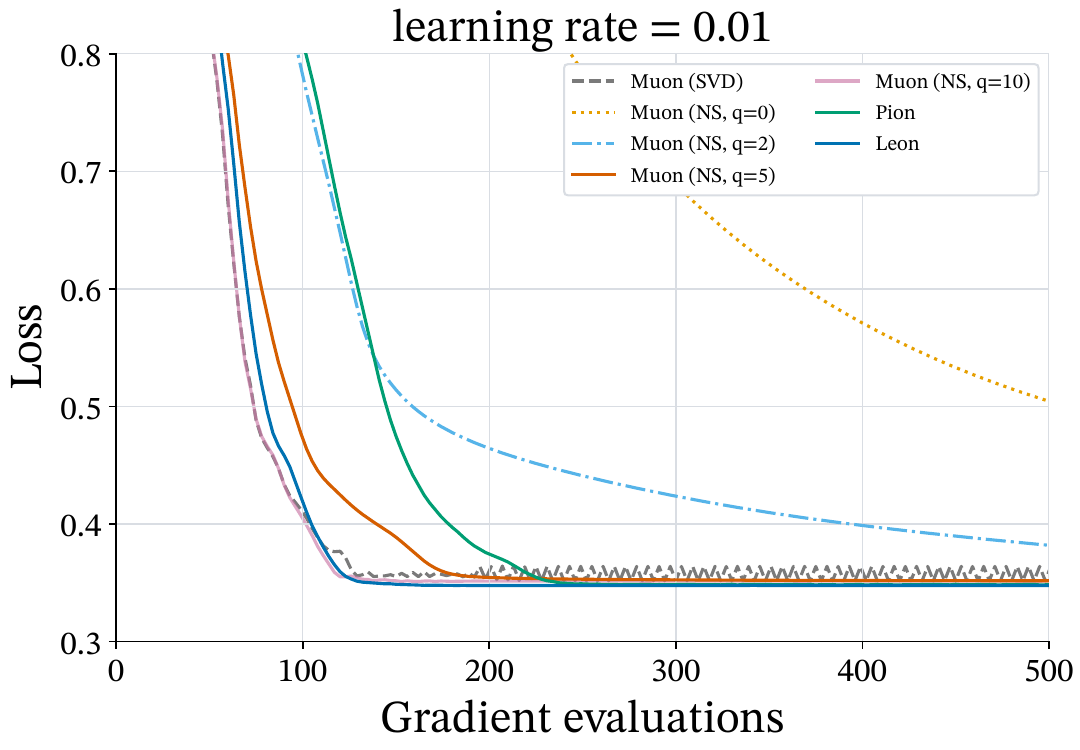}
  \end{minipage}
  \hfill
  \begin{minipage}[t]{0.49\linewidth}
    \centering
    \includegraphics[width=\linewidth]{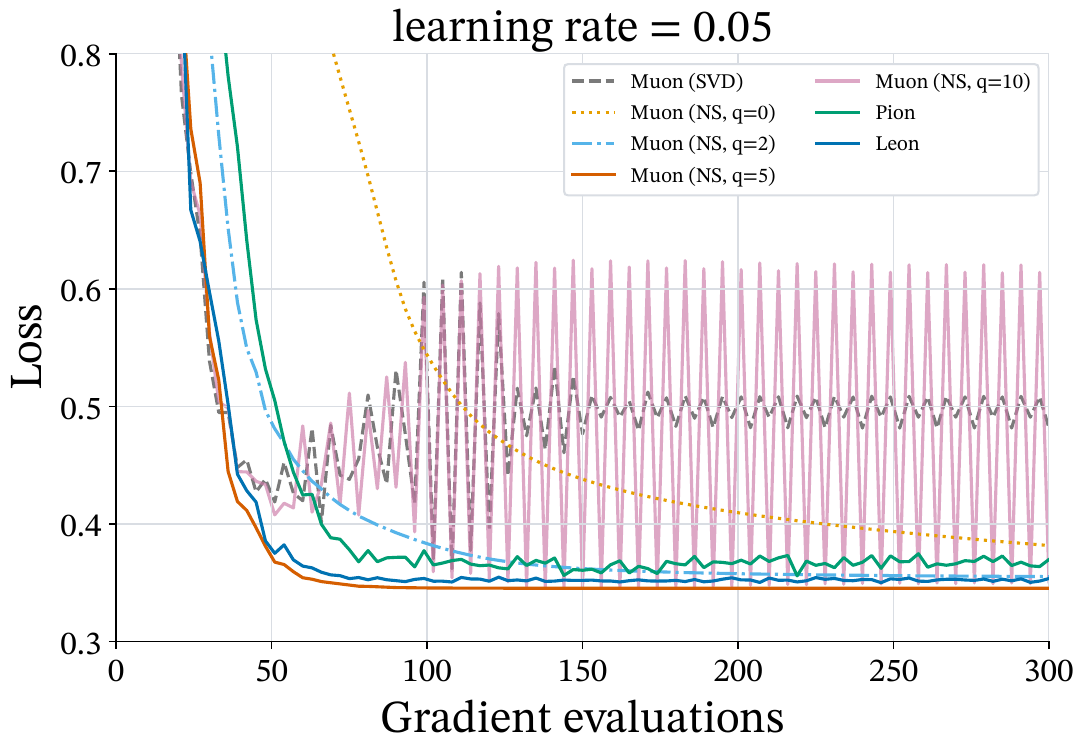}
  \end{minipage}
  \\[3ex]
  \begin{minipage}[t]{0.49\linewidth}
    \centering
    \includegraphics[width=\linewidth]{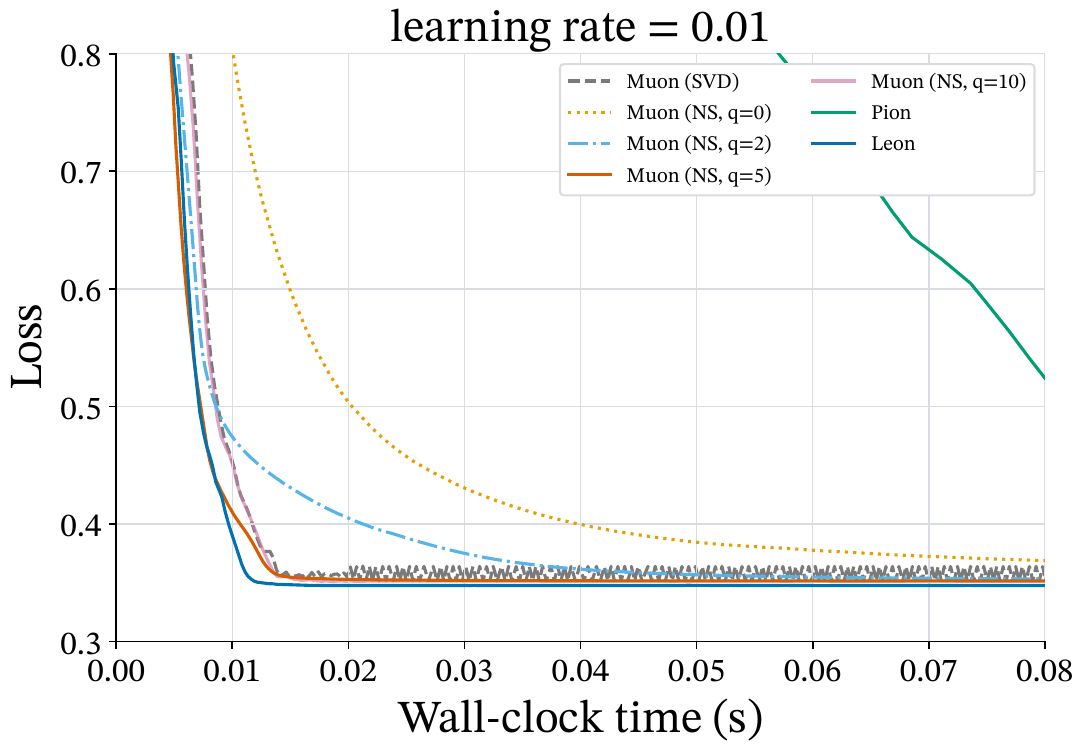}
  \end{minipage}
  \hfill
  \begin{minipage}[t]{0.49\linewidth}
    \centering
    \includegraphics[width=\linewidth]{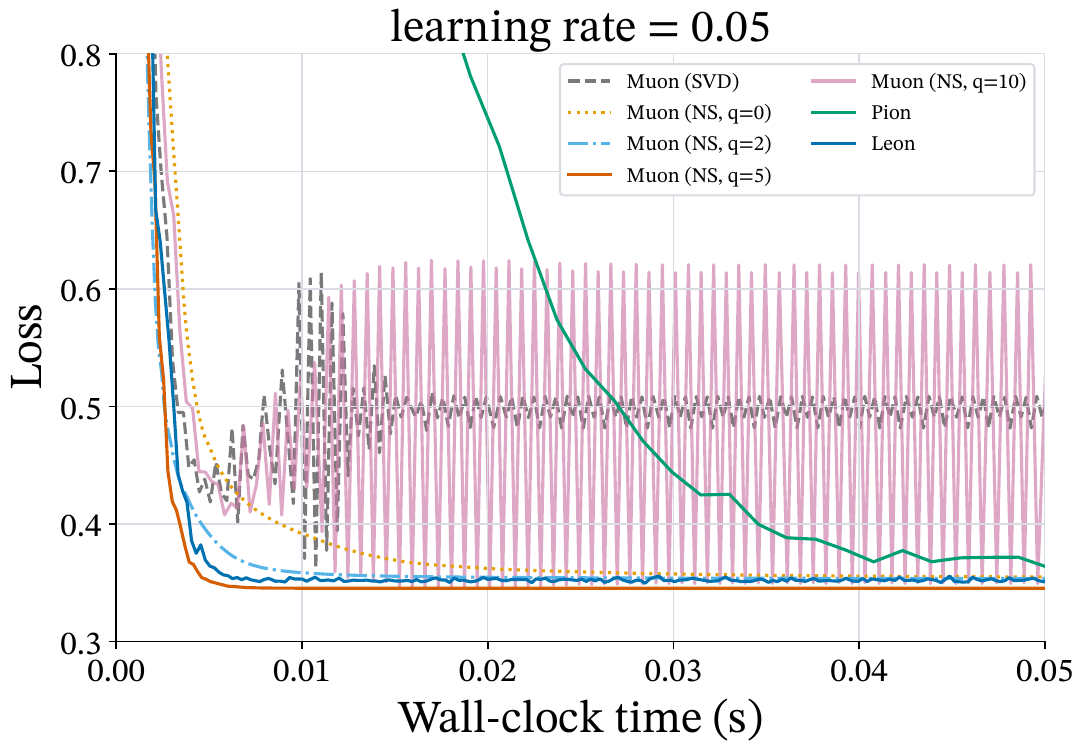}
  \end{minipage}
  \caption{Loss versus gradient evaluations (top) and wall-clock time (bottom) with learning rates $0.01$ (left) and $0.05$ (right).}
  \label{fig:anisotropic-experiment}
\end{figure}

\Cref{fig:anisotropic-experiment} shows how the loss evolves for each depth $q$ and learning rate.
At the larger learning rate, the exact-polar update continues to fluctuate around a high loss value, and the largest depth $q=10$ also exhibits this instability.
On the other hand, the shallowest depth $q=0$ is stable but slow, needing many more gradient evaluations to make comparable progress.
The intermediate depths $q=2$ and $q=5$ avoid both of these failures: they descend quickly and remain stable.
This matches qualitatively the penalty--stability tradeoff discussed in~\cref{sec:ns}, where increasing the depth gives a better approximation to the polar map but increases the Lipschitz constant that governs the stability of the update.
In this sense, the experiment demonstrates a benefit of the finite Newton--Schulz iteration beyond approximating the exact polar map.
Pion and Leon are also stable, and Pion is slower in wall-clock time since each of its steps computes several perturbed polar factors.

\end{document}